\documentclass[10pt]{article}
\usepackage{amsmath,amssymb,bm}
\usepackage{amsthm}
\usepackage{mathtools}
\usepackage[noend]{algorithmic}
\usepackage[ruled,vlined]{algorithm2e}
\usepackage{url}
\usepackage{makeidx}
\usepackage{booktabs,multirow,threeparttable}

\usepackage{graphicx,float,psfrag,epsfig,caption}
\usepackage[usenames,dvipsnames,svgnames,table]{xcolor}
\definecolor{darkgreen}{rgb}{0.0,0,0.9}
\usepackage[pagebackref,letterpaper=true,colorlinks=true,pdfpagemode=none,citecolor=OliveGreen,linkcolor=BrickRed,urlcolor=BrickRed,pdfstartview=FitH]{hyperref}
\usepackage{epstopdf}
\usepackage{color}
\usepackage{xr}
\usepackage{subfigure}
\usepackage{caption}
\usepackage{graphicx}
\usepackage[utf8]{inputenc}
\usepackage{tikz}
\usepackage[english]{babel}
\usepackage{am}

\usepackage{scalerel,stackengine}

\usepackage[mathscr]{euscript}

\DeclareSymbolFont{rsfs}{U}{rsfs}{m}{n}
\DeclareSymbolFontAlphabet{\mathscrsfs}{rsfs}

\stackMath
\newcommand\reallywidehat[1]{%
\savestack{\tmpbox}{\stretchto{%
  \scaleto{%
    \scalerel*[\widthof{\ensuremath{#1}}]{\kern.1pt\mathchar"0362\kern.1pt}%
    {\rule{0ex}{\textheight}}
  }{\textheight}%
}{2.4ex}}%
\stackon[-6.9pt]{#1}{\tmpbox}%
}
\usepackage[mathscr]{euscript}

\DeclareSymbolFont{rsfs}{U}{rsfs}{m}{n}
\DeclareSymbolFontAlphabet{\mathscrsfs}{rsfs}

\numberwithin{equation}{section}

\newtheoremstyle{myexample} 
    {\topsep}                    
    {\topsep}                    
    {\rm }                   
    {}                           
    {\bf }                   
    {.}                          
    {.5em}                       
    {}  

\makeatletter

\newcommand*{\rom}[1]{\expandafter\@slowromancap\romannumeral #1@}
\makeatother

\usetikzlibrary{arrows,shapes}
\begin{document}

\title{\bf Fast Learning in Reproducing Kernel Kre\u{\i}n Spaces via Signed Measures}

\author{Fanghui Liu\thanks{Department of Electrical Engineering
		(ESAT-STADIUS), KU Leuven}, \;\;\;\;  Xiaolin Huang\thanks{Institute of Image Processing and Pattern Recognition, Institute of Medical Robotics, Shanghai Jiao Tong University}, \;\;\;\; Yingyi Chen\footnotemark[1], \;\;\;\; 
Johan A.K. Suykens\footnotemark[1]}

\maketitle

\begin{abstract}
	In this paper, we attempt to solve a long-lasting open question for non-positive definite (non-PD) kernels in machine learning community: can a given non-PD kernel be decomposed into the difference of two PD kernels (termed as positive decomposition)?
	We cast this question as a distribution view by introducing the \emph{signed measure}, which transforms positive decomposition to measure decomposition: a series of non-PD kernels can be associated with the linear combination of specific finite Borel measures.
	In this manner, our distribution-based framework provides a sufficient and necessary condition to answer this open question.
	Specifically, this solution is also computationally implementable in practice to scale non-PD kernels in large sample cases, which allows us to devise the first random features algorithm to obtain an unbiased estimator.
	Experimental results on several benchmark datasets verify the effectiveness of our algorithm over the existing methods.
\end{abstract}


\section{Introduction}
\vspace{-0.1cm}

Devising a pairwise similarity/dissimilarity function plays a significant role in metric learning and kernel learning \cite{Kulis2013Metric,Jain2017Learning,ye2019fast}.
However, such function is not always positive definite (PD) in practice.
For example, we are often faced with \emph{indefinite} (real, symmetric, but not positive definite) kernels including the hyperbolic tangent kernel \cite{smola2001regularization,cho2019large} and truncated $\ell_1$-distance kernel \cite{huang2017classification}. 
Interestingly, some common-used PD kernels, e.g., polynomial kernels, Gaussian kernels, would degenerate to indefinite ones in some cases.
An intuitive example is that a linear combination of PD kernels with negative coefficients \cite{Cheng2005Learning}.
Polynomial kernels using $\ell_2$-normalization (i.e., distributed on the unit sphere) are not always PD \cite{pennington2015spherical}.
Gaussian kernels with some geodesic distances cannot be ensured positive definite \cite{Jayasumana2013Kernel,Feragen2015Geodesic}.
We refer to a survey \cite{Schleif2015Indefinite} for details.

Learning with indefinite similarity/dissimilarity functions is typically modeled in Reproducing Kernel Kre\u{\i}n Spaces (RKKS) \cite{Cheng2004Learning}, where the (reproducing) indefinite kernel can be decomposed into the difference of two PD kernels, a.k.a, positive decomposition \cite{bognar1974indefinite}.
A series of work \cite{roth2003optimal,Ga2016Learning,oglic18a,oglic2019scalable,saha2020learning,liu2020analysis} rely on the positive decomposition.
It is important to note that, indefinite kernel matrices can be decomposed in the difference of two positive semi-definite matrices by eigenvalue decomposition, but
{\emph{for a given indefinite kernel, does it admit a positive decomposition?}} is a long-lasting open question in machine learning community.
In fact, it appears non-trivial how to verify that an indefinite kernel can be associated with RKKS except for some intuitive examples, e.g., a linear combination of PD kernels with negative coefficients.
In the past, we usually assume that a (reproducing) indefinite kernel is associated with RKKS in practice while the theoretical gap cannot be ignored. 
In particular, the used eigenvalue decomposition in indefinite kernel based algorithms \cite{Ga2016Learning,oglic18a,oglic2019scalable} often incurs huge computational and space complexities and thus is infeasible to large-scale problems.

To answer the open question, we consider indefinite kernel in a distribution view.
Our model is based on the \emph{signed measure}, which generalizes Borel measure to be negative.
Accordingly, the positive decomposition can be transformed to measure decomposition and thus we provide a sufficient and necessary condition to answer this question.
Our distribution-based framework is simple but effective, which naturally allows us to devise unbiased random features based algorithm to scale indefinite kernel methods in large sample cases.
Formally, we make the following contributions:
\begin{itemize}
	\item In Section~\ref{sec:rfmsm}, by introducing the signed measure, we provide a sufficient and necessary condition to answer the above open question for indefinite kernels via the measure decomposition technique. Moreover, this condition also guides us how to find a specific positive decomposition in practice, and thus we can devise a unbiased estimator to obtain randomized feature maps.
	To the best of our knowledge, this is the first work to generate unbiased estimation for non-PD kernel approximation by random features.
	\item In Section~\ref{sec:dot}, we demonstrate that spherically dot-product kernels including polynomial kernels, arc-cosine kernels, and the popular NTK in two-layer ReLU network on the unit sphere\footnote{We use the $\ell_2$ normalization scheme to ensure the data on the unit sphere as suggested by \cite{pennington2015spherical}, which is different from directly using spherically i.i.d data.}, are radial and non-PD associated with RKKS. Then we demonstrate the feasibility of our random feature algorithm on several indefinite kernels admitting the positive decomposition. 
	\item In Section~\ref{sec:exp}, we evaluate various non-PD kernels on several typical benchmark datasets to validate the effectiveness of our algorithm.
\end{itemize}

\section{Preliminaries and Related Works}
In this section, we briefly sketch some basic ideas of RKKS \cite{bognar1974indefinite} and Bochner's theorem in random features \cite{rahimi2007random}, and then introduce related works on indefinite kernel approximation.

\subsection{Reproducing Kernel Kre\u{\i}n Spaces}

Here we briefly review on the Kre\u{\i}n spaces and the reproducing kernel Kre\u{\i}n space (RKKS).
Detailed expositions can be found in book \cite{bognar1974indefinite}.
Most of the readers would be familiar with Hilbert spaces.
Kre\u{\i}n spaces share some properties of Hilbert spaces but differ in some key aspects which we shall emphasize as follows.

Kre\u{\i}n spaces are indefinite inner product spaces endowed with a Hilbertian topology.
\begin{definition}\label{definiterkks}
	(Kre\u{\i}n space \cite{bognar1974indefinite}) An inner product space is a Kre\u{\i}n space $\mathcal{H_K}$ if there exist two Hilbert spaces $\mathcal{H}_+$ and $\mathcal{H}_-$ such that\\
	i) $\forall f \in \mathcal{H_K}$, it can be decomposed into $f=f_+ \oplus f_-$, where $f_+ \in \mathcal{H}_+$ and $f_- \in \mathcal{H}_-$, respectively.\\
	ii) $\forall f,g \in \mathcal{H_K}$, $\langle f,g\rangle_{\mathcal{H_K}}=\langle f_+,g_+\rangle_{\mathcal{H}_+} - \langle f_-,g_-\rangle_{\mathcal{H}_-}$.
\end{definition}
The Kre\u{\i}n space $\mathcal{H_K}$ can be decomposed into a direct sum $\mathcal{H_K} = \mathcal{H_+} \oplus \mathcal{H_-}$.
Besides, the inner product on $\mathcal{H_K}$ is non-degenrate, i.e., for $f \in \mathcal{H_K}$, if $\langle f,g\rangle_{\mathcal{H_K}} = 0$ for any $g \in \mathcal{H_K}$, we have $f = 0$.
From the definition, the decomposition $\mathcal{H_K} = \mathcal{H}_+ \oplus \mathcal{H}_-$ is not necessarily unique. For a fixed decomposition, the inner product $ \langle f, g \rangle_{\mathcal{H_K}}$ is given accordingly \cite{Ga2016Learning,oglic18a}.
The key difference from Hilbert spaces is that the inner products might be negative for Kre\u{\i}n spaces, i.e., there exists $f \in \mathcal{H_K}$ such that $\langle f,f\rangle_{\mathcal{H_K}} < 0$.
If $\mathcal{H}_+$ and $\mathcal{H}_-$ are two RKHSs, the Kre\u{\i}n space $\mathcal{H_K}$ is an RKKS associated with a unique indefinite reproducing kernel $k$ such that the reproducing property holds, i.e., $\forall f \in \mathcal{H_K},~f(x) = \langle f,k(x,\cdot) \rangle_{\mathcal{H_K}}$.
\begin{proposition}
	(positive decomposition \cite{bognar1974indefinite})
	Let $k: \mathbb{R}^d \times \mathbb{R}^d \rightarrow \mathbb{R}$ be a real-valued kernel function. Then there exists an associated RKKS identified with a reproducing kernel $k$ if and only if $k$ admits a positive decomposition $k = k_+ - k_-$, where $k_+$ and $k_-$ are two positive definite kernels.
\end{proposition}
From the definition, this decomposition is not necessarily unique.
As mentioned before, not every indefinite kernel function admits a representation as a difference between two positive definite kernels.

\subsection{Bochner's theorem and random features}
A positive definite function corresponds to a nonnegative and finite Borel measure, i.e.,~a probability distribution, via Fourier transform by the following theorem. 
\begin{theorem}[Bochner's Theorem \cite{bochner2005harmonic}]
	\label{bochner}
	Let $k: \mathbb{R}^d \times \mathbb{R}^d \rightarrow \mathbb{R}$ be a bounded continuous function satisfying the stationary property, i.e., $k(\bm x, \bm x') = k(\bm x - \bm x')$. Then, $k$ is
	positive definite if and only if it is the (conjugate) Fourier transform of a nonnegative and finite Borel measure $\mu$ (rescale it to a probability measure by setting $k(\bm 0)=1$)
	\begin{equation*}
	k(\bm x - \bm x') \!=\! \int_{\mathbb{R}^d} \!\!  e^{   \mathrm{i} {\bm \omega}^{\!\top}( \bm x - \bm x')} \mu(\mathrm{d} {\bm \omega}) = \mathbb{E}_{\bm \omega \sim \mu} \big[ e^{\mathrm{i}{\bm \omega}^{\!\top}(\bm x - \bm x')}\big] \,.
	\end{equation*}
\end{theorem}
Typically, the kernel in practical uses is real-valued and thus the imaginary part can be discarded, i.e., $k(\bm x - \bm x') = \mathbb{E}_{\bm \omega \sim \mu} \cos[\bm \omega^{\!\top} (\bm x - \bm x')]$.
Accordingly, we can use the Monte Carlo method to sample a series of random features $\{ \bm \omega_i \}_{i=1}^s$ from the distribution $\mu$ to approximate the PD kernel function $k$, a.k.a. Random Fourier features (RFF) \cite{rahimi2007random}.
It brings promising performance and solid theoretical guarantees on scaling up kernel methods in classification \cite{sun2018but}, nonlinear component analysis \cite{xie2015scale,lopez2014randomized}, and neural tangent kernel (NTK) \cite{jacot2018neural}.
Improvements on RFF mainly focus on variance reduction by advanced sampling methods, e.g., quasi-Monte Carlo sampling \cite{yang2014quasi}, Monte Carlo sampling with orthogonal constraints \cite{Yu2016Orthogonal,lyu2017spherical,choromanski2017unreasonable}, leverage-score sampling \cite{avron2017random,li2019towards}, and quadrature based methods \cite{dao2017gaussian,munkhoeva2018quadrature,liu2020towards}, see a survey \cite{liu2020survey} for details.

\subsection{Signed measure}
Let $\mu: \mathcal{A} \rightarrow [0, +\infty]$ be a measure on a set $\Omega$ satisfying $\mu(\emptyset)=0$ and $\sigma$-additivity (i.e., countably additive). We call $\mu$ a finite measure if $\mu(\Omega) < +\infty$. Specifically, $\mu$ is a probability measure if $\mu(\Omega)=1$, and the triple $(\Omega, \mathcal{A}, \mu)$ is referred as the corresponding probability space. Here we consider the signed measure, a generalized version of a measure allowing for negative values.

\begin{definition}\label{defsign} (Signed measure \cite{athreya2006measure})
	Let $\Omega$ be some set, $\mathcal{A}$ be a $\sigma$-algebra of subsets on $\Omega$. A \emph{signed measure} is a  function $\mu: \mathcal{A} \rightarrow [-\infty, +\infty)~\mbox{or}~(-\infty, +\infty]$ satisfying $\sigma$-additivity.
\end{definition}
Based on this definition, the following theorem shows that any signed measure can be represented by the difference of two nonnegative measures.
\begin{theorem} (Jordan decomposition \cite{athreya2006measure})
	\label{jorthem}
	Let $\mu$ be a signed measure defined on the $\sigma$-algebra $\mathcal{A}$ as given in Definition~\ref{defsign}. There exists two (nonnegative) measures $\mu_+$ and $\mu_-$ (one of them is a finite measure) such that $\mu = \mu_+ - \mu_-$.
\end{theorem}
The total mass of $\mu$ on $\mathcal{A}$ is defined as $\| \mu \| = \| \mu_+ \| + \| \mu_- \|$.
Note that this decomposition is not unique.

\subsection{Related works}
Learning with indefinite kernels in RKKS can be solved by eigenvalue transformation \cite{chen2009learning,Ying2009Analysis}, stabilization \cite{Cheng2004Learning,Ga2016Learning}, and minimization \cite{oglic18a}.
However, these methods need eigenvalue decomposition and cannot be directly applied to large-scale problems.

To scale indefinite kernel matrices in large sample problems, Nystr\"{o}m approximation works in a data-dependent way, and is a good choice to seek a low-rank representation to approximate indefinite kernel matrices, e.g., \cite{oglic2019scalable,mehrkanoon2018indefinite,schleif2016probabilistic}.
Besides, Liu et al. \cite{liu2019double} decompose (a subset of) kernel matrix into two PD kernel matrices, and then learn their respective randomized feature maps by infinite Gaussian mixtures. However, this approach in fact focuses on approximating kernel matrices rather than kernel functions.
If we consider indefinite kernel approximation by random features in a data-independent way, Pennington et al. \cite{pennington2015spherical} find that the polynomial kernel using $\ell_2$-normalized data is not PD, and then use (positive) mixtures of Gaussian distributions, associated with a PD kernel, to approximate it.
This is in essence using a PD kernel to approximate an indefinite one.
Till now, approximating non-PD kernels by random features cannot ensure unbiased and has not yet been fully investigated.
In this paper, our work provides an unbiased estimator without extra parameters, so as to achieve both simplicity and effectiveness.

Besides, our algorithm can be also applied to dot-product kernels with $\ell_2$-normalized data, e.g., polynomial kernels on the unit sphere.
Recent works for polynomial kernel approximation include Maclaurin expansion \cite{kar2012random}, the tensor sketch technique \cite{Pham2013Fast,meister2019tight}, and oblivious subspace embedding \cite{avron2014subspace,woodruff2020near}.

\section{Model}
\label{sec:rfmsm}

In this section, by introducing the concept of signed measures \cite{athreya2006measure}, we attempt to answer the \emph{open question} and then devise the sampling strategy for random features.
For notational simplicity, we denote $z\coloneqq\| \bm z \|_2 = \| \bm x -\bm x' \|_2$ and $\omega\coloneqq\| \bm \omega \|_2$.
Moreover, a function $k(\bm z)$ is called \emph{radial} if $k(\bm z) = k(\| \bm z\|_2)$. To notify, the considered stationary kernels in this paper are all \emph{radial}, and their Fourier transforms are also \emph{radial}, i.e., $\mu(\omega) = \mu(\bm \omega)$, refer to \cite{pennington2015spherical,boisbunon2012class}.

\subsection{Answer to the open question in RKKS}

As mentioned before, not every indefinite kernel admits a representation as a difference between two positive definite kernels. In fact we do not know how to verify that an indefinite kernel can be associated with RKKS except for some intuitive examples, e.g., a linear combination of PD kernels with negative coefficients.
By virtue of measure decomposition of the signed measure in Theorem~\ref{jorthem}, we provide a sufficient and necessary condition in the following theorem to answer the question in RKKS: \emph{for a given indefinite kernel, does it admit a positive decomposition?}
\begin{theorem}\label{thm:fft}
	Assume that an indefinite kernel is stationary, i.e., $k(\bm x, \bm x') = k(\bm x - \bm x')$.
	Denote its (generalized) Fourier transform as the measure $\mu$, then we have the following results:\\
	(i) {\bf Existence}: $k$ admits the positive decomposition, i.e., $k=k_+ - k_-$, if and only if the total mass of the measure $\mu$ is finite, i.e., $\| \mu \| < \infty$. 
	Here $k_+$ and $k_-$ are two reproducing kernels associated with two reproducing kernel Hilbert spaces (RKHS) $\mathcal{H}_+$ and $\mathcal{H}_-$, respectively. \\
	(ii) {\bf Representation}: If $\| \mu \| < \infty$, we choose $\mu_+$ and $\mu_-$ such that $\mu = \mu_+ - \mu_-$, then the associated RKHSs $\mathcal{H}_{\pm}$ are given by
	\begin{equation*}
	\mathcal{H}_{\pm} = \left\{ f: \| f \|^2_{\mathcal{H}_{\pm}} =  \int_{\mathbb{R}^{d}} \frac{|F(\bm \omega)|^{2}}{\mu_{\pm}(\bm \omega)} \mathrm{d} \bm \omega < \infty \right\} \,,
	\end{equation*}
	where $F(\bm \omega)$ is the Fourier transform of $f$.
\end{theorem}
\begin{proof}
	The proof can be found in Appendix~\ref{app:fft}.
\end{proof}
{\bf Remark:} We provide an explicit sufficient and necessary condition to link the Jordan decomposition of signed measures to positive decomposition in RKKS. 
\\
\textit{(i)} Functions that can be written as a difference between two positive functions have been studied and characterized intrinsically in the field of harmonic analysis \cite{maserick1975BV,alpay1991somecon,alpay1991some}.
They partly answered this question either restricted in the one-dimensional case (i.e., $d=1$) or assuming the indefinite kernel $k(\bm x, \bm x')$ to be jointly analytic of $\bm x$ and $\bm x'$ in a neighborhood of the origin.
However, the used univariate condition could not be directly applied to the machine learning society and the smoothness requirement on kernels excludes some non-differentiable kernels, e.g., arc-cosine kernels. 
Instead, Theorem~\ref{thm:fft} provides an access via Fourier transform to verify whether a (reproducing) indefinite kernel belongs to RKKS or not. The measure decomposition is much easier to be founded than positive decomposition in RKKS that cannot be verified in practice.\\
\textit{(ii)} Theorem~\ref{thm:fft} can be further improved to cover some non-squared-integrable kernel functions, e.g., conditionally positive definite kernels \cite{wendland2004scattered}, of which the standard Fourier transform does not exist.
In this case, Theorem~\ref{thm:fft} needs to consider the Fourier transform in Schwartz space \cite{donoghue2014distributions}.
For example, conditionally positive kernels correspond to a positive Borel measure $\mu$ on $\mathbb{R}^d \backslash\{\bm 0\}$ with an analytic function in Schwartz space, refer to \cite[Theorem 2.3]{sun1993conditionally}.

\subsection{Randomized feature map}

\newcounter{mytempeqncnt}
\begin{figure*}[!t]
	\normalsize
	\begin{equation}\label{lcintrep}
	\begin{split}
	k(\bm x - \bm x') &=\! c_1\!\! \int_{\mathbb{R}^d} \!  e^{  \mathrm{i} {\bm \omega}^{\!\top} \! \bm z } \mu_+(\mathrm{d} {\bm \omega}) \!-\! c_2 \!\! \int_{\mathbb{R}^d} \! e^{ \mathrm{i} {\bm \nu}^{\!\top} \! \bm z } \mu_-(\mathrm{d} {\bm \nu}) = c_1 \| \mu_+\| \mathbb{E}_{\bm \omega \sim \tilde{\mu}_+} \big[\cos(\bm \omega^{\!\top} \bm z) \big]  - c_2 \| \mu_-\| \mathbb{E}_{\bm \nu \sim \tilde{\mu}_-} \big[ \cos(\bm \nu^{\!\top}\bm z) \big]\\
	&\coloneqq  k_+(\bm x - \bm x') -  k_-(\bm x - \bm x') \approx  \tilde{k}_+(\bm x - \bm x') -  \tilde{k}_-(\bm x - \bm x') \\
	& = \frac{1}{s}\sum_{i=1}^s \langle \mbox{Re}[\varphi_i(\bm x)], \mbox{Re}[\varphi_i(\bm x')]  \rangle - \frac{1}{s}\sum_{i=1}^s \langle \mbox{Im}[\varphi_i(\bm x)], \mbox{Im}[\varphi_i(\bm x')]  \rangle\,.
	\end{split}
	\end{equation}
	\vspace*{-0.5cm}
\end{figure*}

The condition in Theorem~\ref{thm:fft} serves as a guidance for us to find a specific positive decomposition in practice.
Hence we are ready to develop our random feature algorithm for (real-valued) non-PD kernels.
One intuitive implementation way is choosing $\mu_+ := \max\{ \mu, 0 \}$ and $\mu_- := \min \{ 0, \mu \}$ such that $\mu = \mu_+ - \mu_-$. Then the stationary indefinite kernel $k$ can be expressed by Eq.~\eqref{lcintrep} via $k =  k_+ - k_-$ with two positive constants $c_1$, $c_2$. The decomposed two Borel measures $\tilde{\mu}_+\coloneqq\mu_+/\| \mu_+\|,~\tilde{\mu}_-\coloneqq\mu_-/\| \mu_-\|$ are associated with two (normalized) PD kernels $k_+$ and $k_-$, respectively.
Accordingly, the stationary indefinite kernel $k$ can be approximated by $k \approx \tilde{k} := \tilde{k}_+ - \tilde{k}_-$.
That is, $k(\bm x, \bm x') = \mathbb{E}_{\bm \omega, \bm \nu} \left\langle \Phi(\bm x), \Phi(\bm x') \right\rangle \!\approx \! \tilde{k}(\bm x, \bm x') := \frac{1}{s} \! \sum_{i=1}^s \left\langle \varphi_i(\bm x), \varphi_i(\bm x') \right\rangle$ where $\Phi(\bm x)$ is the explicit feature mapping $\Phi(\bm x) = [\varphi_1(\bm x), \cdots, \varphi_s(\bm x)]^{\!\top}$ with $\varphi_i(\bm x)$
\begin{equation}\label{realimg}
\begin{split}
& \varphi_i(\bm x) \!=\!\! \Big[\sqrt{c_1\| \mu_+\|}\cos(\bm \omega_i^{\!\top} \bm x), \sqrt{c_1\| \mu_+\|}\sin(\bm \omega_i^{\!\top} \bm x),  \quad \mathrm{i}\sqrt{c_2\| \mu_-\|} \! \cos(\bm \nu_i^{\!\top} \bm x), \mathrm{i}\sqrt{c_2\| \mu_-\|}\sin(\bm \nu_i^{\!\top} \bm x)\Big]\,,
\end{split}
\end{equation}
where $\mbox{i}$ is the imaginary unit. Then random features are obtained by sampling $\{\bm \omega_i \}_{i=1}^s \sim \mu_+/\| \mu_+\|$ and $\{\bm \nu_i \}_{i=1}^s \sim \mu_-/\| \mu_-\|$.
The employed sampling method can be Monte Carlo sampling, orthogonal Monte Carlo sampling  \cite{Yu2016Orthogonal,choromanski2019unifying}, or leverage-score based sampling \cite{avron2017random,bach2017equivalence}.
The real and imaginary part in $\varphi_i(\bm x)$ correspond to $k_+$ and $k_-$, and thus our estimation is unbiased.
It is important to note that, we need the imaginary unit in the feature mapping due to the difference operation\footnote{A simple example is that $a-b=\langle ( \sqrt{a}, \mathrm{i}\sqrt{b}) , (\sqrt{a}, \mathrm{i}\sqrt{b} ) \rangle$ for two nonnegative real numbers $a,b$.}, and then the approximated kernel is still real-valued.

The complete random features process is summarized in Algorithm~\ref{alg:one1}.
For a given stationary indefinite kernel, its $\mu$, $\mu_{\pm}$ can be pre-computed, which is independent of the training data. 
In this way, our algorithm achieves the same complexity with the standard RFF by $\mathcal{O}(ns^2)$ time and $\mathcal{O}(ns)$ memory.
Besides, the formulation in Eq.~\eqref{lcintrep}, as well as Algorithm~\ref{alg:one1}, is general enough to cover various PD and non-PD kernels.
Stationary PD kernels admit Eq.~\eqref{lcintrep} by choosing $c_1=1$ and $c_2=0$ where we have $\mu = \mu_+$ associated with $\| \mu \| =1$, i.e., a Borel measure. 
Hence, the Bochner's theorem can be regarded as a special case of the considered integration representation~\eqref{lcintrep} in this paper. 

The approximation performance in our method for indefinite kernels still achieves theoretical guarantees with those of PD kernels by the following proposition. The result can be easily derived from \cite{sutherland2015error,choromanski2018geometry}, of which the proof refers to Appendix~\ref{app:error} for completeness. 
\begin{proposition}\label{prop:error}
	Let $k$ be a stationary indefinite kernel in RKKS with two Borel measures $\tilde{\mu}_{\pm}$ defined in Eq.~\eqref{lcintrep}, we have the following results:\\
	(i) {\bf Approximation:} Let $\mathcal{S}_R$ be the compact ball by $\mathcal{S}_R = \{ \Delta | \Delta \in \mathbb{R}^d, \| \Delta \|_2 \leq R \}$, then given the approximated kernel $\tilde{k}$ obtained by our algorithm via Monte Carlo sampling, for any $\epsilon > 0$
	\begin{equation*}
	\operatorname{Pr}\!\left[ \sup_{\bm x, \bm x' \in \mathcal{S}_R} \! |k(\bm x, \bm x') - \tilde{k}(\bm x, \bm x')| \geq \epsilon \right] \leq 66  \left(  \frac{2\sigma R}{\epsilon}  \right)^{\!2} \!\!\! e^{-\frac{s\epsilon^2}{32(d+2)}},
	\end{equation*}
	where $\sigma^2 := \mathbb{E}_{\bm \omega \sim \tilde{\mu}_{+}}[{\bm \omega}^{\!\top} {\bm \omega}] + \mathbb{E}_{\bm \omega \sim \tilde{\mu}_{-}}[{\bm \omega}^{\!\top} {\bm \omega}] < \infty$.\\
	(ii) {\bf Variance reduction:} If we consider orthogonal Monte Carlo (OMC) sampling \cite{Yu2016Orthogonal,choromanski2019unifying} in our algorithm, it admits $\operatorname{MSE}[\tilde{k}^{\operatorname{OMC}}(\bm x, \bm x')] \leq \operatorname{MSE}[\tilde{k}^{\operatorname{MC}}(\bm x, \bm x')]$ for sufficiently large $d$, where the mean-squared error (MSE) is defined as $\mathbb{E}[\tilde{k}(\bm x, \bm x')] = \mathbb{E}[k(\bm x, \bm x') - \tilde{k}(\bm x, \bm x')]$.
\end{proposition}

\begin{algorithm}[tb]
	\SetAlgoNoLine
	\KwIn{A kernel function $k(\bm x, \bm x')= k(z)$ with $z\coloneqq\| \bm x - \bm x'\|_2$ and the number of random features $s$.}
	\KwOut{Random feature map $\Phi(\cdot): \mathbb{R}^d \rightarrow \mathbb{R}^{4s}$ such that $k(\bm x, \bm x') \!\approx \! \frac{1}{s} \! \sum_{i=1}^s \left\langle \varphi_i(\bm x), \varphi_i(\bm x') \right \rangle$.}
	1. Obtain the measure $\mu(\cdot)$ of the kernel $k$ via (generalized) Fourier transform \;
	2. Given $\mu$, let $\mu \coloneqq \mu_+ - \mu_-$ be the Jordan decomposition with two nonnegative measures $\mu_{\pm}$ and compute the total mass $\| \mu\| = \| \mu_+ \| + \| \mu_-\| $\;
	3. Sample $\{ \bm \omega_i \}_{i=1}^s \sim \mu_+/\| \mu_+\|$ and $\{ \bm \nu_{i} \}_{i=1}^s \sim \mu_-/\| \mu_-\|$\;
	4. Output the explicit feature mapping $\Phi(\bm x)$ with $\varphi_i(\bm x)$ given in Eq.~\eqref{realimg}.
	\caption{Random features for various indefinite kernels via generalized measures.}
	\label{alg:one1}
\end{algorithm}

\section{Examples}
\label{sec:dot}

In this section, we investigate a series of indefinite kernels for a better understanding of our random features algorithm.
We begin with an intuitive example, the linear combination of PD kernels with negative constraints.
Then we discuss several dot-product kernels using $\ell_2$ normalization data, including the polynomial kernel \cite{pennington2015spherical}, the arc-cosine kernel \cite{cho2009kernel}, and the NTK kernel in a two-layer ReLU network \cite{bietti2019inductive}.

{\bf A linear combination of positive definite kernels with negative coefficients:} Kernels in this class admit the formulation $k = \sum_{i=1}^{t} a_i k_i$, where $\{ k_i \}_{i=1}^t$ is the set of PD kernels, and $a_i \in \mathbb{R}$.
This is a typical example of indefinite kernels in RKKS, which admits positive decomposition such that $k = k_+ - k_-$ with two PD kernels $k_{\pm}$.
Theorem~\ref{thm:fft} guides us to find $\mu_{\pm}$ based on the sign of $a_i$.
Hence we explicitly decompose an indefinite kernel in this class into the difference of two PD kernels, i.e., $k = k_+ - k_- \coloneqq \sum_{i=1}^{t} \max(0,a_i) k_i - \sum_{i=1}^{t} \max(0,-a_i) k_i$.
Then the corresponding nonnegative measures $\mu_{\pm}$ can be subsequently obtained due to the additivity of Fourier transform. 
We take the Delta-Gaussian kernel \cite{oglic18a} $k(\bm x,\bm x')= \exp(-{\| \bm x - \bm x' \|_2^2}/{2\tau_1^2}) - \exp(-{\| \bm x - \bm x' \|_2^2}/{2\tau_2^2})$ as an example.
This kernel admits $c_1=c_2=1$ and $\|\mu_+\|=\|\mu_-\|=1$ in Eq.~\eqref{lcintrep}, and its random feature mapping is given by Eq.~\eqref{realimg} with $\{ \bm \omega_i \}_{i=1}^s \sim \mathcal{N}(0, \tau_1^{-2}\bm I_d)$ and $\{ \bm \nu_i \}_{i=1}^s \sim \mathcal{N}(0, \tau_2^{-2}\bm I_d)$.

After providing the above simple warming-up example, we now discuss dot-product kernels on the unit sphere, and demonstrate the feasibility of our algorithm.

{\bf Polynomial kernels on the sphere:}  
Pennington et al. \cite{pennington2015spherical} point out that a polynomial kernel on the unit sphere by $\ell_2$ normalization is of $k(\bm x, \bm x') = \left(1-\frac{\| \bm x- \bm x'\|_2^2}{a^2}\right)^p$ for $a \geq 2$ and $p \geq 1$ and $z\coloneqq \| \bm x -\bm x' \|_2 \in [0,2]$.
This kernel is indefinite since its Fourier transform is not a nonnegative measure in \cite{pennington2015spherical}
\begin{equation}\label{polymu}
\mu(\omega) \!=\!\! \sum_{i=0}^{p} \frac{p !}{(p-i) !}\!\!\left(1 \!-\! \frac{4}{a^{2}}\right)^{p-i}\!\! \left(\frac{2}{a^{2}}\right)^{i} \!\! \left(\frac{2}{\omega}\right)^{\frac{d}{2}+i} \!\!\! J_{\frac{d}{2}+i}(2 \omega),
\end{equation}
which results from the oscillatory behavior of the Bessel function of the first kind $J_{d / 2+i}(2 \omega)$.
We demonstrate $\| \mu \| < \infty$ (see in Appendix~\ref{finsigm}), which makes the integration our random features algorithm feasible by decomposing  $\mu = \mu_+ - \mu_-$ with $\mu_+ = \max\{ 0, \mu \}$ and $\mu_- = \min\{ 0, \mu \}$.
Then random feature map for this kernel can be also given by Eq.~\eqref{realimg}
with $\{ \bm \omega_i \}_{i=1}^s \sim \mu_+/\| \mu_+\|$ and $\{ \bm \nu_i \}_{i=1}^s \sim \mu_-/\| \mu_-\|$. 
Therefore, Algorithm~\ref{alg:one1} is suitable for this kernel.
Note that the (scaled) measure $\mu_{\pm}$ is not a typical probability distribution, but the radial property of the Fourier transform allows us to conduct rejection sampling in one dimension to sample from this ``complex" distribution, which does not incur too much computational cost.
We experimentally evaluate this with other sampling schemes in Section~\ref{sec:exp}.
Compared to \cite{pennington2015spherical} using a positive sum of Gaussians to approximate $\mu(\omega)$, where parameters in Gaussians need to be optimized aforehand, our algorithm achieves both simplicity and effectiveness by having \textit{(i)} an unbiased estimator, \textit{(ii)} incurring no extra parameters.
Figure~\ref{everpoly} shows the superiority of our method to SRF on approximating the spherical polynomial kernel $k(z)$. 
It can be found that, our method is unbiased to achieve lower mean squared error since SRF directly overlooks the negative part of the signed measure $\mu$.

Next we consider the NTK of two-layer ReLU networks on the unit sphere\footnote{This setting is actually different from the considered $\ell_2$-normalization case in this paper that cannot ensure the data are i.i.d on the unit sphere.} \cite{bietti2019inductive}.
Since this kernel in fact consists of zero/first-order arc-cosine kernels \cite{cho2009kernel}, we combine them together for discussion.

{\bf NTK of Two-layer ReLU networks on the unit sphere:} Bietti and Mairal \cite{bietti2019inductive} consider a two-layer ReLU network of the form $f(\bm x; \bm \theta) = \sqrt{2}{s}\sum_{j=1}^s \sum_{j=1}^s a_j \max\{ \bm \omega_j^{\!\top} \bm x,0 \}$, with the parameter $\bm \theta = (\bm \omega_1^{\!\top},\cdots,\bm \omega_s^{\!\top}, a_1,\cdots,a_s)$ initialized according to $\mathcal{N}(0,1)$. By formulating ReLU as $\max\{ \bm \omega_j^{\!\top} \bm x,0 \} = (\bm \omega_j^{\!\top} \bm x)_+$, we have the following formulation corresponding to NTK \cite{bietti2019inductive,chizat2019lazy}
\begin{equation}\label{ntktr}
\begin{aligned}
& k\left(\bm x, \bm x^{\prime}\right)  = 2 \mathbb{E}_{\bm \omega \sim \mathcal{N}(\bm 0, \bm I)}\!\left[(\bm \omega^{\top}\! \bm x)_{+}(\bm \omega^{\!\top}\! \bm x^{\prime})_{+}\right]\\
& \!+\! 2\left(\bm x^{\top}\! \bm x^{\prime}\right) \mathbb{E}_{\bm \omega \sim \mathcal{N}(\bm 0, \bm I)}\!\left[1\left\{\bm \omega^{\top}\! \bm x \!\geq\! 0\right\} \! 1\!\left\{\bm \omega^{\top}\! \bm x^{\prime} \!\geq\! 0\right\}\right],
\end{aligned}
\end{equation}
which can be further represented by $k\left(\bm x, \bm x^{\prime}\right) = \|\bm x\|\left\|\bm x^{\prime}\right\|\cdot \kappa \, \bigl(\left\langle \bm x, \bm x^{\prime}\right\rangle/ ( \|\bm x\|\left\|\bm x^{\prime}\right\|)\!\bigr)$ with $\kappa(u)\coloneqq u\kappa_0(u) + \kappa_1 (u)$. Here, $\kappa_0(u) = 1 - \frac{1}{\pi} \arccos(u)$ corresponds to the zero-order arc-cosine kernel and $\kappa_{1}(u)=\frac{1}{\pi}(u(\pi-\arccos (u))+\sqrt{1-u^{2}})$ is the first-order arc-cosine kernel \cite{cho2009kernel}.
Furthermore, such NTK kernel is proved to be stationary but indefinite by the following theorem.

\begin{theorem}\label{thentk}
	For any $\bm x, \bm x' \in \mathbb{R}^d$, by $\ell_2$ normalization, the NTK kernel of a two layer ReLU network of the form $f(\bm x; \bm \theta) = \sqrt{2}{s}\sum_{j=1}^s \sum_{j=1}^s a_j \max\{ \bm \omega_j^{\!\top} \bm x,0 \}$ is stationary, that is,
	\begin{equation*}
	k(\bm x, \bm x') =\frac{2-z^2}{\pi} \arccos\left(\frac{1}{2}z^2 -1\right) + \frac{z}{2\pi} \sqrt{4-z^2} \,,
	\end{equation*}
	where $z\coloneqq\| \bm x- \bm x'\|_2 \in [0,2]$.
	However, the function $k(z), z \in [0,2]$ is not positive definite.\footnote{The behavior of $k(z)$ with $z >2$ is undefined. Following \cite{pennington2015spherical}, we set $k(z)=0$ for $z>2$.}
\end{theorem}

\begin{proof}
	The proof can be found in Appendix~\ref{app:ftntk}.
\end{proof}
Since the above NTK on the unit sphere can be formulated as $k(\bm x, \bm x') = \langle \bm x, \bm x' \rangle \kappa_0(\bm x, \bm x') + \kappa_1(\bm x, \bm x')$ associated with arc-cosine kernels, we have the direct corollary for arc-cosine kernels.
\begin{corollary}\label{coroarc}
	The zero/first order arc-cosine kernel is not positive definite if the data are $\ell_2$-normalized, and its measure $\mu$ is given in Appendix~\ref{meazoa}.
\end{corollary}

{\bf Remark:} These spherical dot-product kernels including polynomial kernels, arc-cosine kernels, and NTK are indefinite by $\ell_2$ normalization, which extends the classical insight on spherical dot-product kernels via spherical harmonics \cite{muller2006spherical}.
Besides, our findings motivate us to scrutinize functional spaces, the approximation performance, and generalization properties of over-parameterized networks in deep learning theory if considering $\ell_2$-normalization data, which in return expands the usage scope of indefinite kernels.

In Appendix~\ref{meazoa}, we compute the measure $\mu$ of arc-cosine kernels, which is quite complex as it involves with the sum of infinite series with Bessel functions.
When taking finite series (e.g., one term) as an approximation, we demonstrate $\| \mu \| < \infty$.
In this case, Algorithm~\ref{alg:one1} is accordingly suitable for arc-cosine kernels and NTK on the unit sphere.
If we take more terms in finite series, the calculation appears non-trivial.
And specifically, there exists a gap between the original $\mu$ and its approximation by finite series, so we do not include these two kernels in our experiments.

\section{Experiments}
\label{sec:exp}
We evaluate the proposed method  on four representative benchmark datasets including \emph{letter}\footnote{\url{https://archive.ics.uci.edu/ml/datasets.html.}}, \emph{ijcnn1}\footnote{\url{https://www.csie.ntu.edu.tw/~cjlin/libsvmtools/datasets/}\label{web}}, \emph{covtype}\textsuperscript{\ref {web}}, and \emph{cod-RNA}\textsuperscript{\ref {web}}, see in Table~\ref{tablarge}.
The datasets are normalized to $[0, 1]^d$ by an $\ell_2$-norm scaling scheme and have been given with training/test partition except for \emph{covtype}. 
Hence, we randomly split the \emph{covtype} dataset into the training and test sets by half.
In our experiment, the used indefinite kernels are the spherical polynomial kernel $k(\bm x, \bm x') = \left(1-{\| \bm x- \bm x'\|_2^2}/{a^2}\right)^p$ with $a=2,~p=2$ in \cite{pennington2015spherical}, and the Delta-Gaussian kernel $k(\bm x,\bm x')= \exp(-{\| \bm x - \bm x' \|^2}/{2\tau_1^2}) - \exp(-{\| \bm x - \bm x' \|^2}/{2\tau_2^2})$ with $\tau_1=1$ and $\tau_2 = 10$ in \cite{oglic18a}.
The compared algorithms include SRF (Spherical Random Features) \cite{pennington2015spherical}, DIGMM (Double-Infinite Gaussian Mixtures Model) \cite{liu2019double}, and Nystr\"{o}m with leverage score \cite{oglic2019scalable}.
Moreover, we also include Random Maclaurin (RM) \cite{kar2012random}, Tensor Sketch (TS) \cite{Pham2013Fast}, and Tensorized Random Projection (TRP) \cite{meister2019tight} for polynomial kernel approximation. Note that the related error bars and standard deviations are obtained by running the experiments for $10$ times.
All experiments are implemented in MATLAB and carried out on a PC with Intel$^\circledR$ i7-8700K CPU (3.70 GHz) and 64 GB RAM. The source code of our implementation can be found in \url{http://www.lfhsgre.org}.

\begin{figure}[t] 
	\centering 
	\subfigure[graph of kernels]{
		\includegraphics[width=0.43\textwidth]{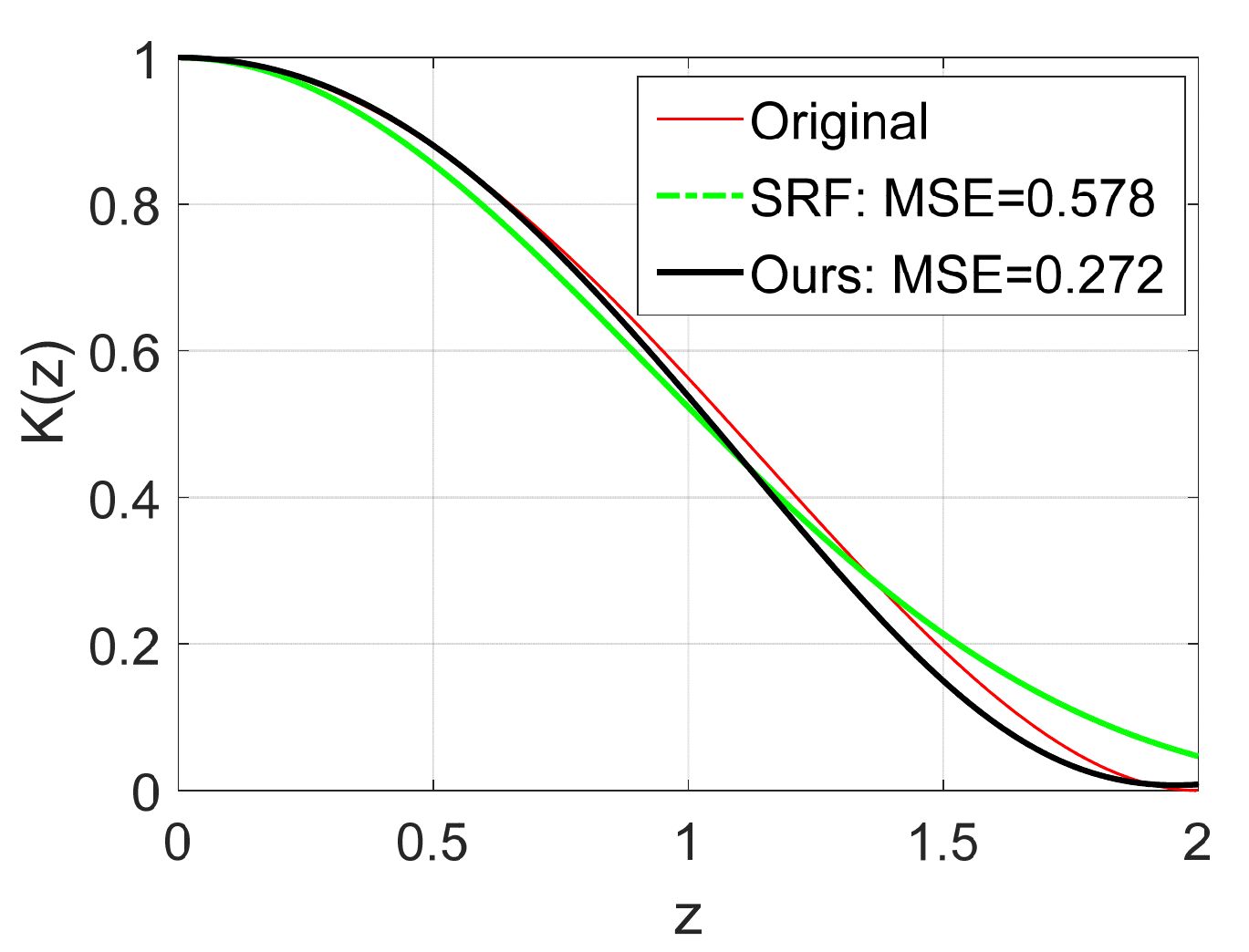}}
	\subfigure[signed measure $\mu$]{\label{everpolymu}
		\includegraphics[width=0.41\textwidth]{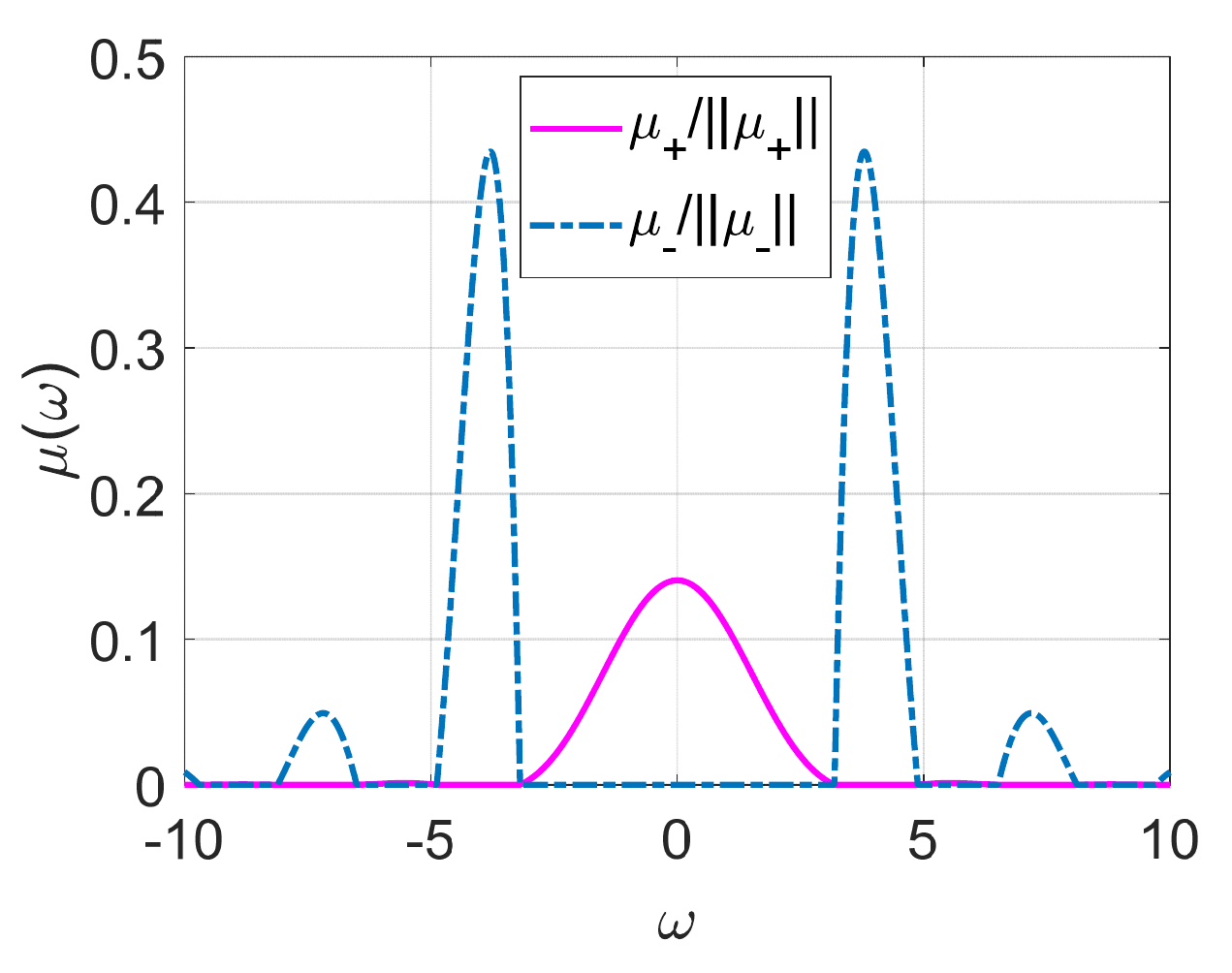}}
	\caption{Approximation of the spherical polynomial kernel with $a=p=2$.}\label{everpoly}
\end{figure}

\begin{table}
	\centering
	\begin{threeparttable}
		\caption{Benchmark datasets.} 
		\label{tablarge}
		\begin{tabular}{cccccccccccccc}
			{\bf Datasets}  & $d$  & \#training & \#test \\
			\midrule
			\emph{letter} & 16 & 12,000 & 6,000 \\
			\emph{ijcnn1} &22 &49,990 &91,701 \\
			\emph{covtype}  &54 &290,506 &290,506
			\\
			\emph{cod-RNA}  &8 &59,535 &157,413 \\
			\bottomrule
		\end{tabular}
	\end{threeparttable}
\end{table}

\begin{figure*}[!htb]
	\centering
	\subfigure[\emph{letter}]{
		\includegraphics[width=0.23\textwidth]{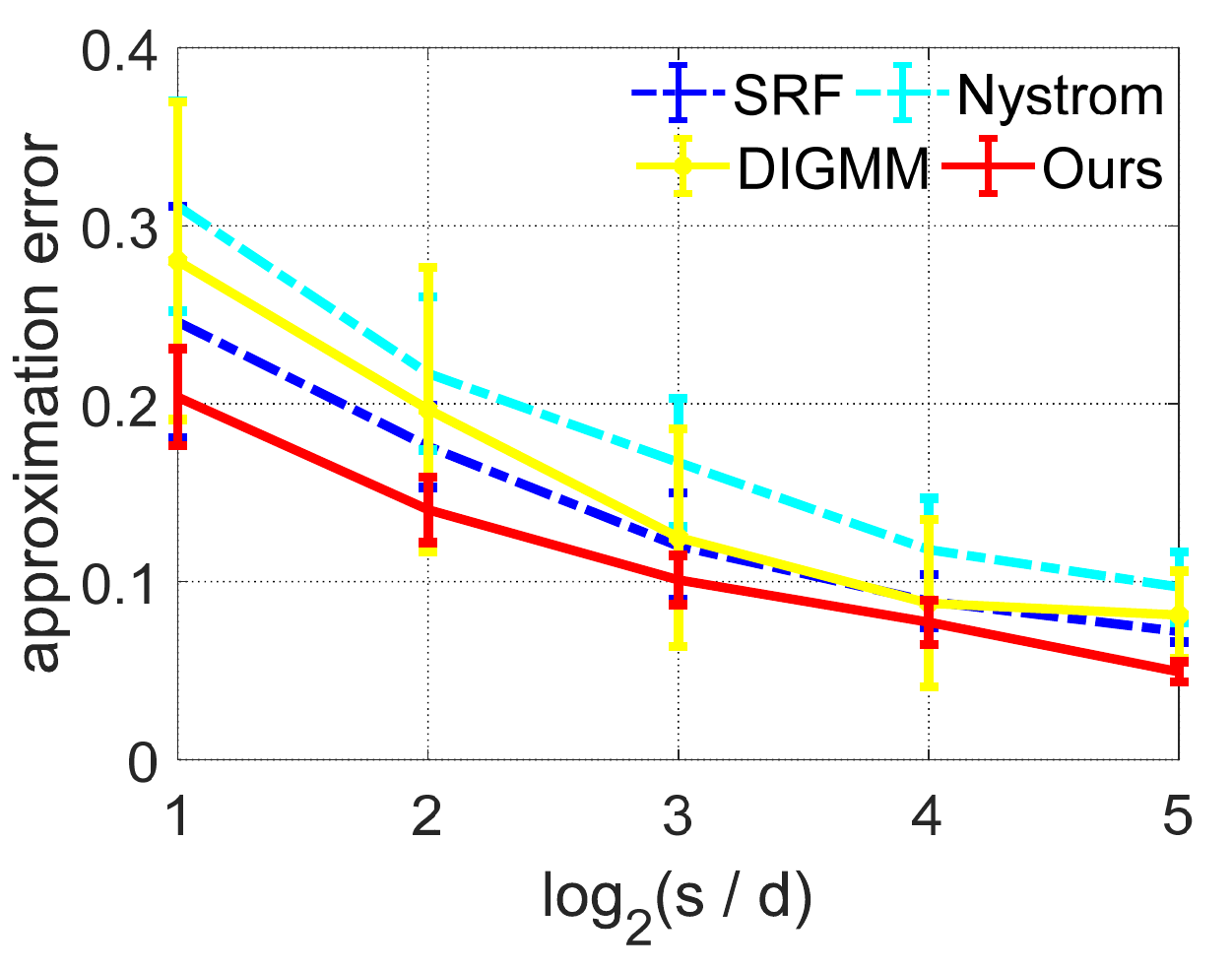}}
	\subfigure[\emph{ijcnn1}]{
		\includegraphics[width=0.23\textwidth]{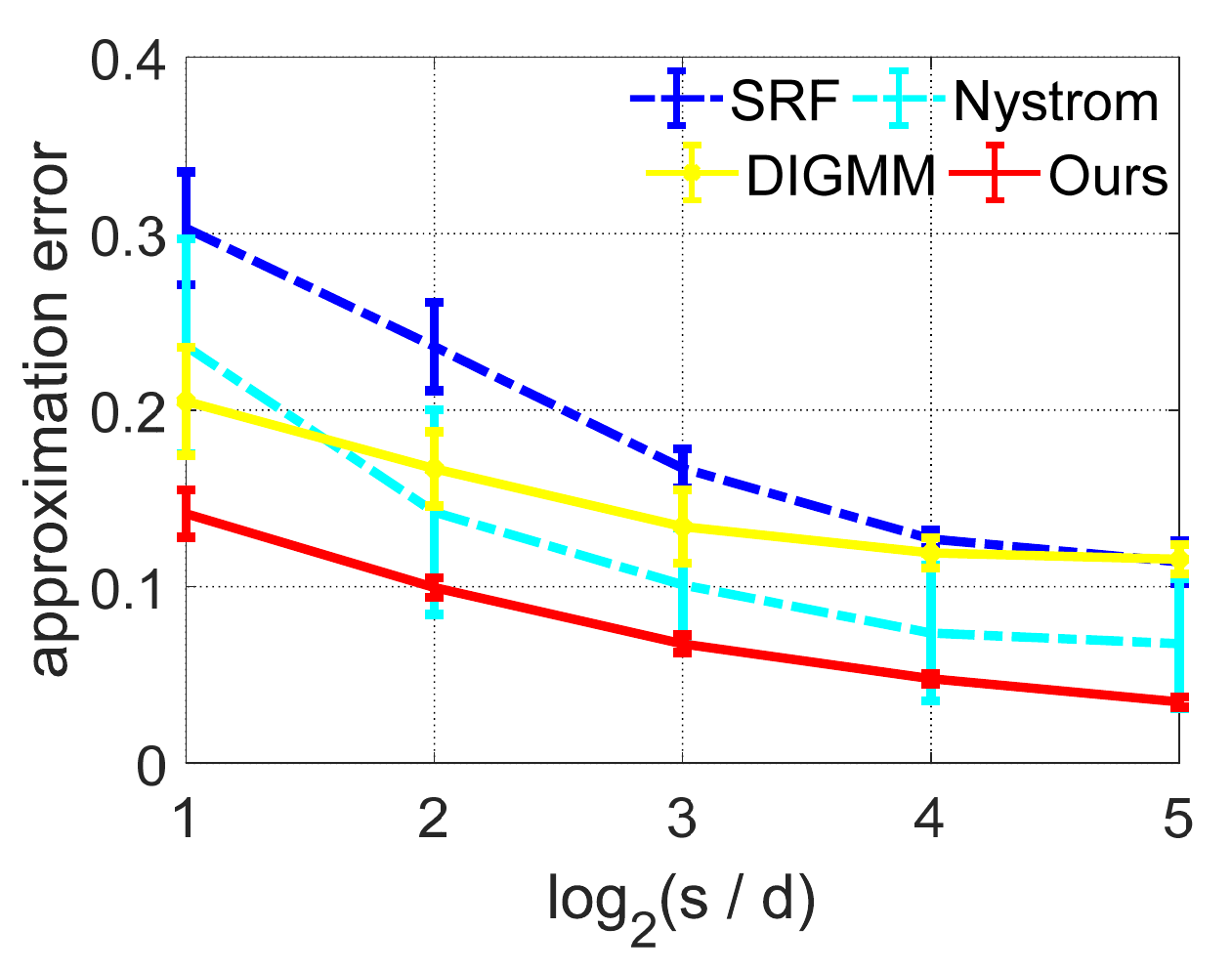}}
	\subfigure[\emph{covtype}]{
		\includegraphics[width=0.232\textwidth]{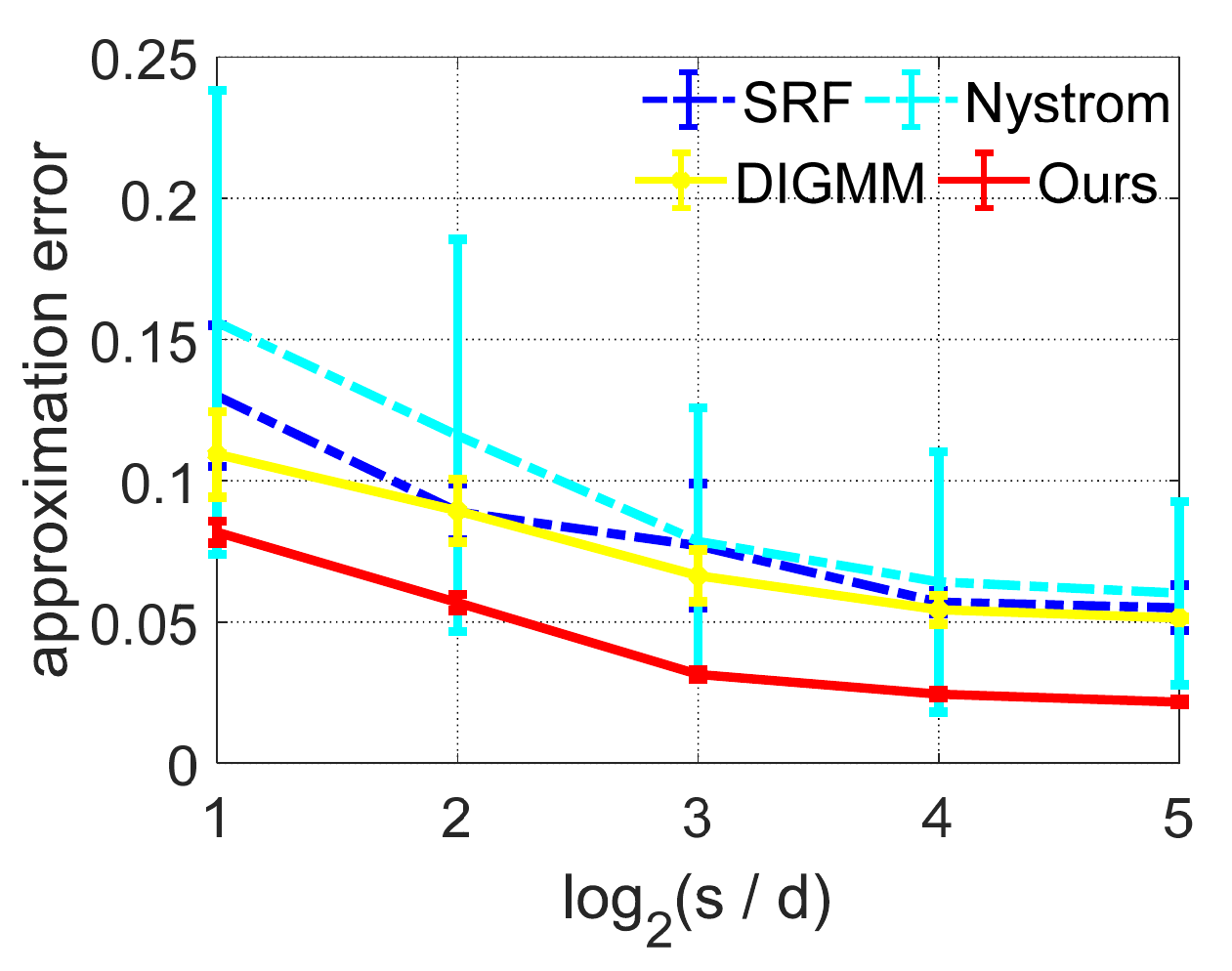}}
	\subfigure[\emph{cod-RNA}]{
		\includegraphics[width=0.232\textwidth]{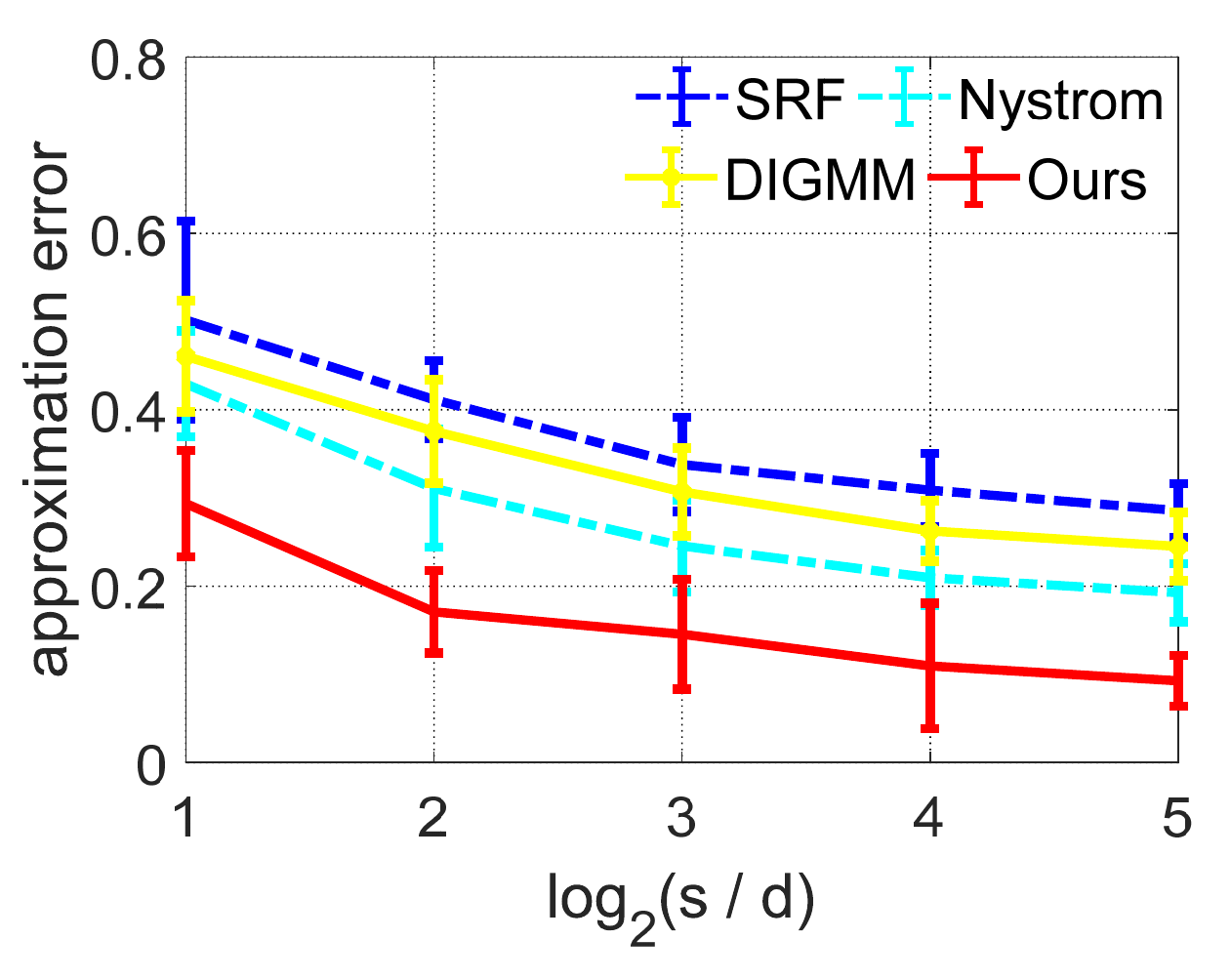}}

	\centering
	\subfigure{
		\includegraphics[width=0.23\textwidth]{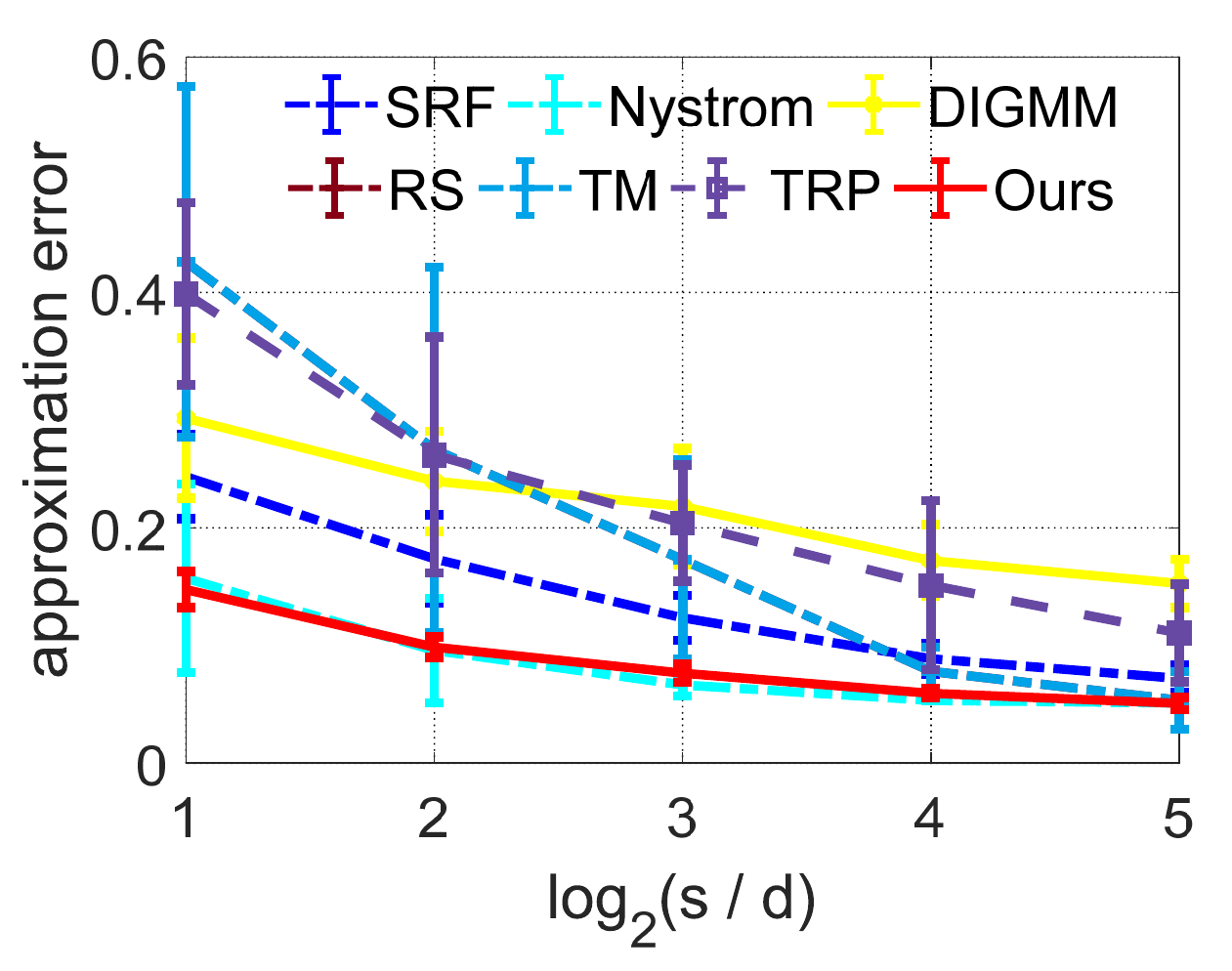}}
	\subfigure{
		\includegraphics[width=0.23\textwidth]{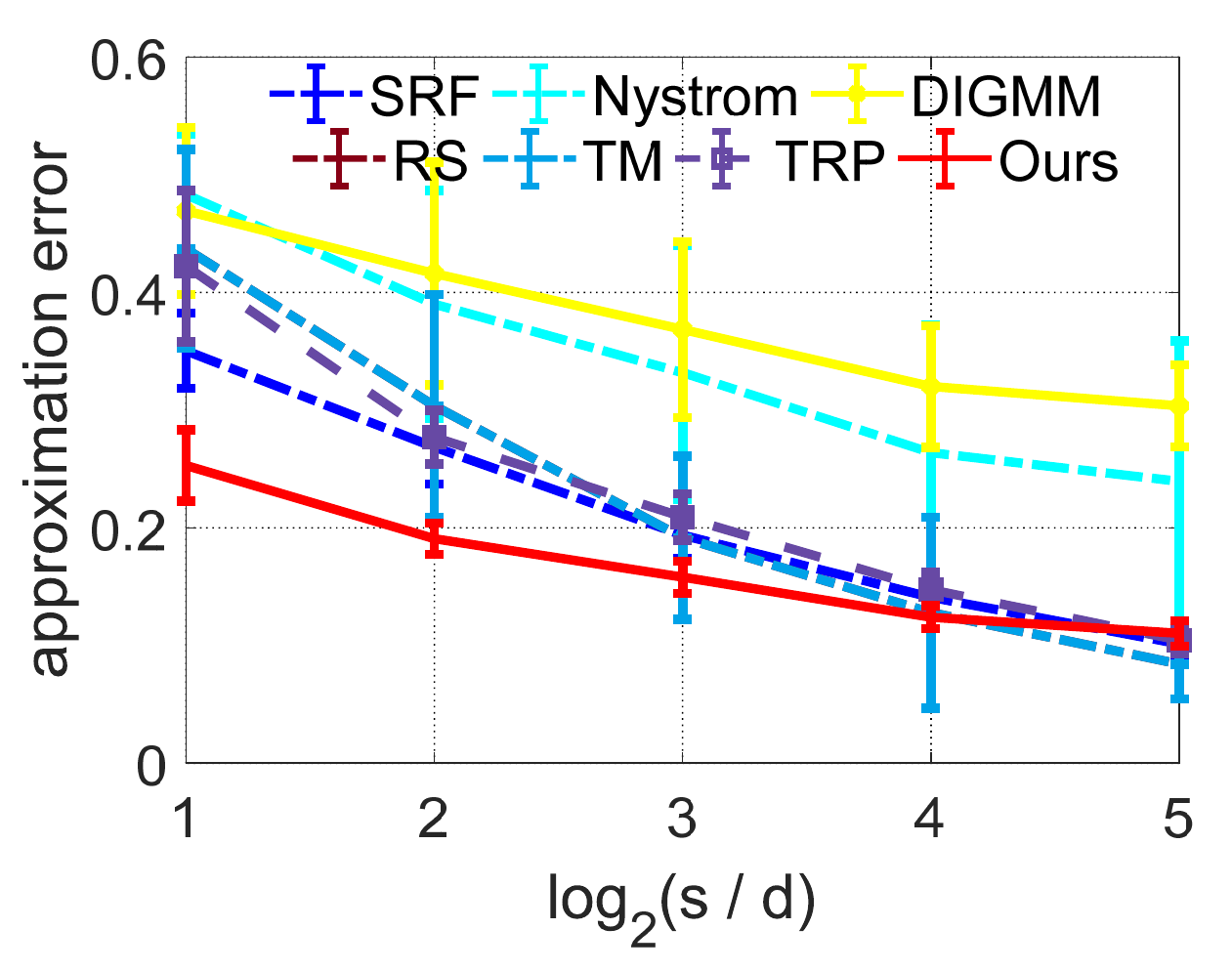}}
	\subfigure{
		\includegraphics[width=0.232\textwidth]{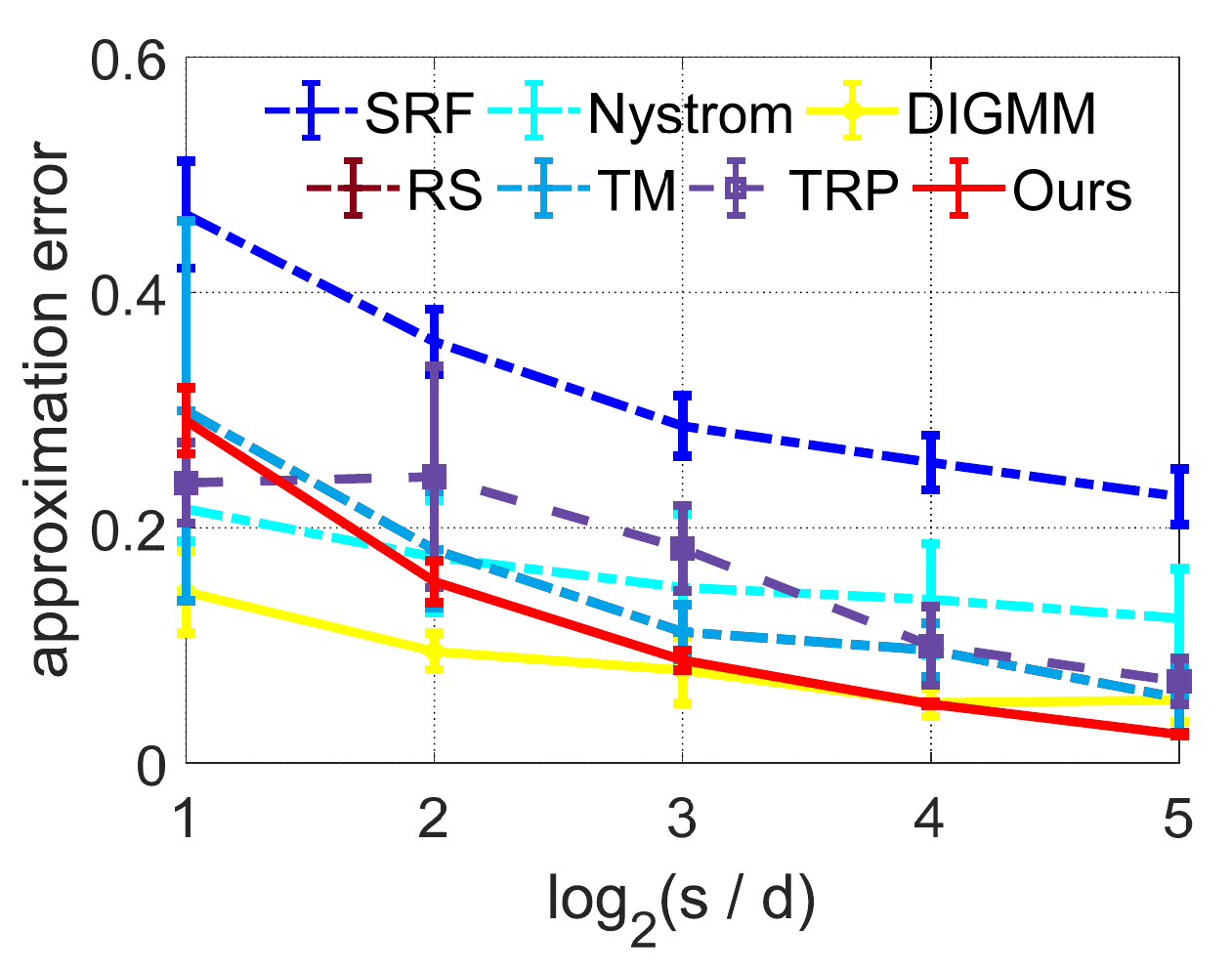}}
	\subfigure{
		\includegraphics[width=0.232\textwidth]{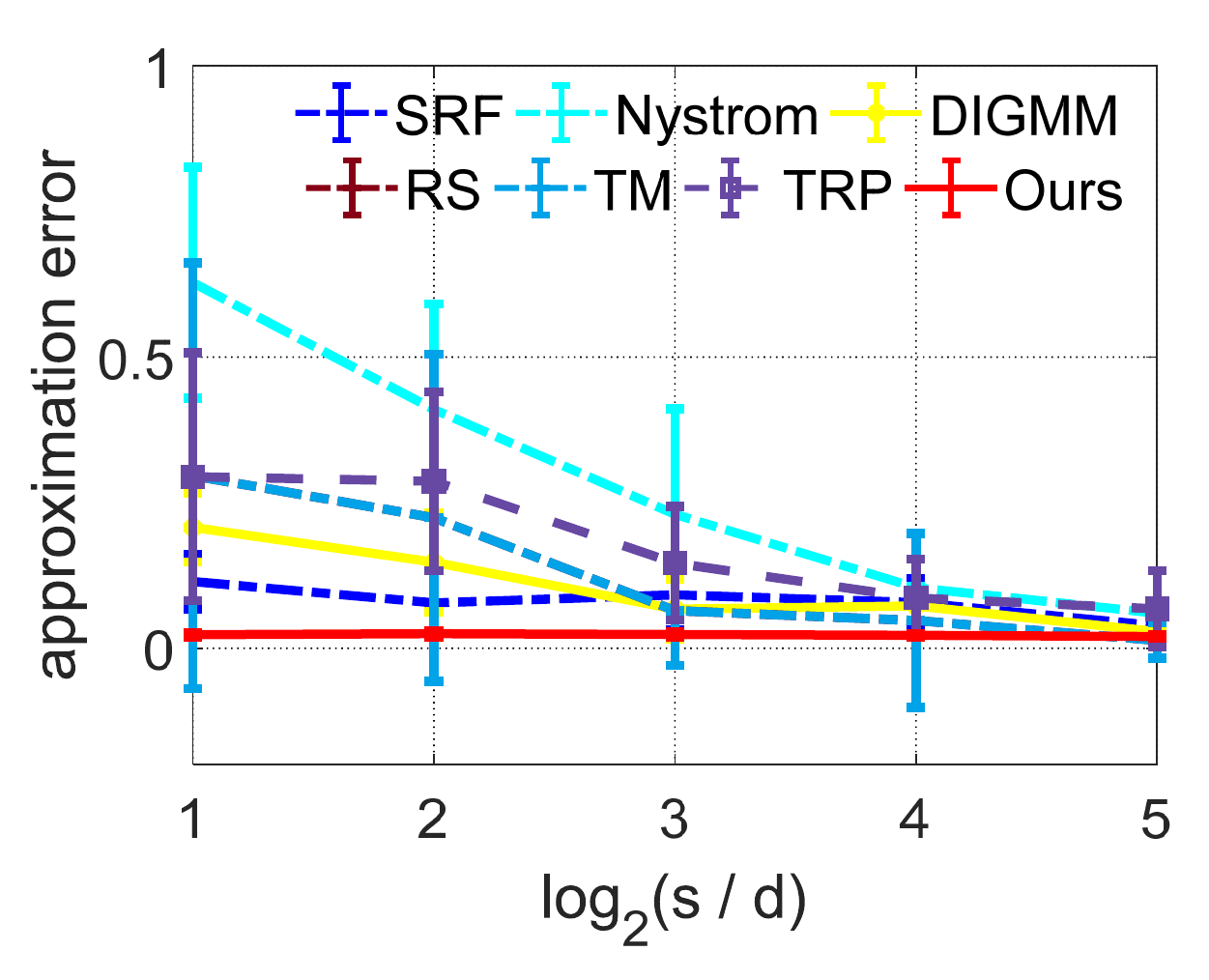}}
	
	\caption{Comparisons of various algorithms for approximation error across the Delta-Gaussian kernel (top) and the spherical polynomial kernel (bottom) on four datasets.}	\label{figapp}
	\vspace{-0.15cm}
\end{figure*}

\begin{figure*}[!htb]
	
	\centering
	\subfigure[\emph{letter}]{
		\includegraphics[width=0.234\textwidth]{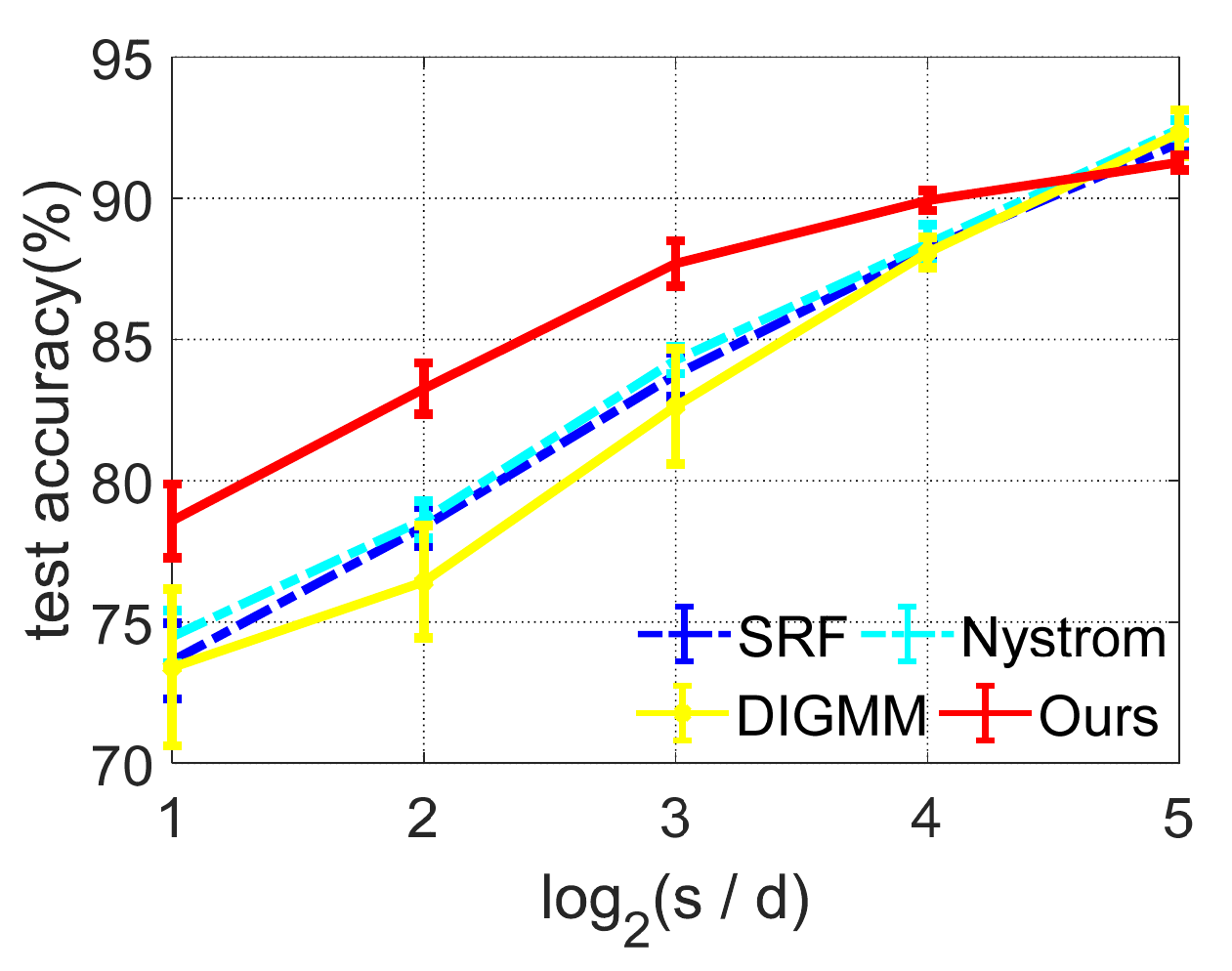}}
	\subfigure[\emph{ijcnn1}]{
		\includegraphics[width=0.234\textwidth]{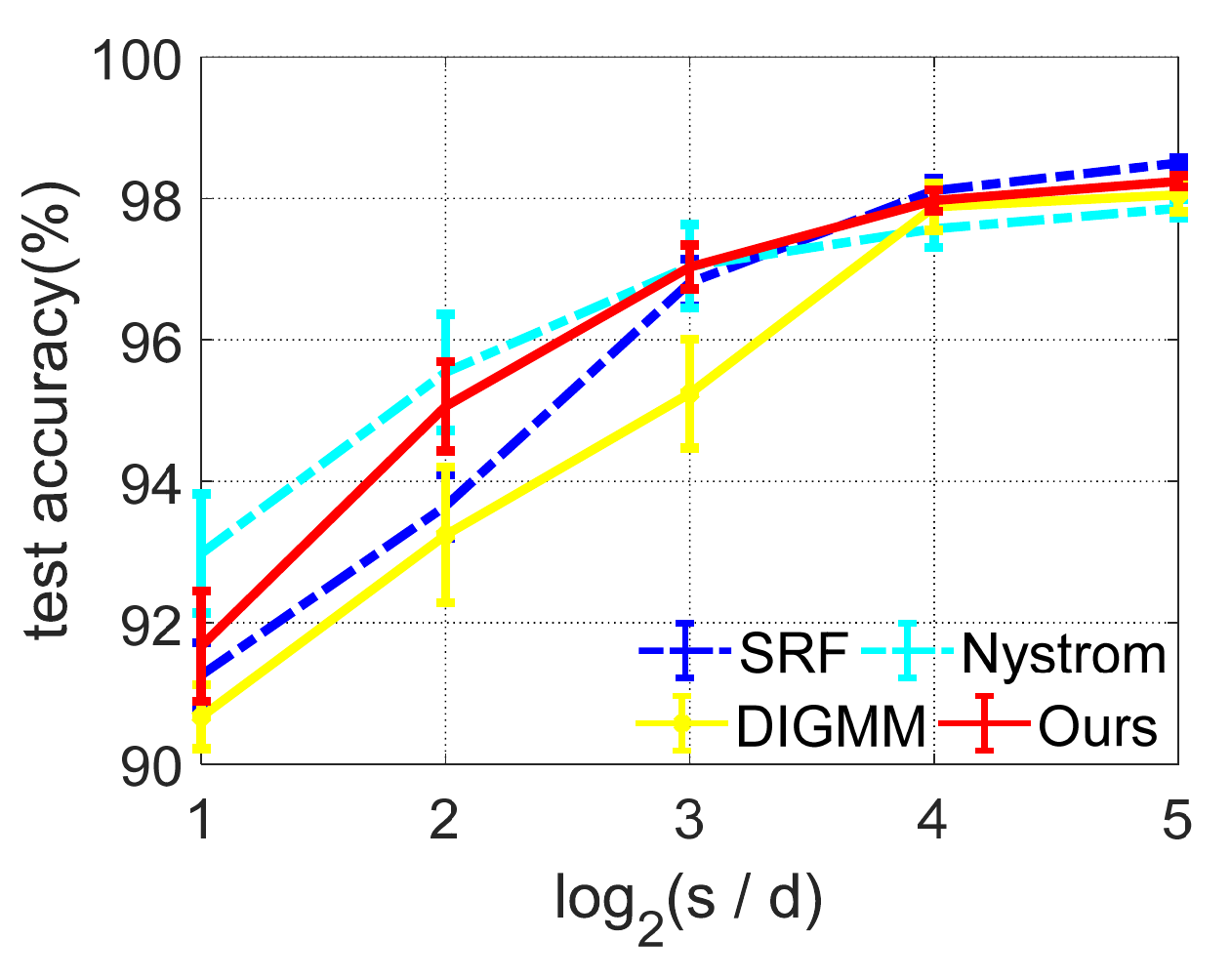}}
	\subfigure[\emph{covtype}]{
		\includegraphics[width=0.234\textwidth]{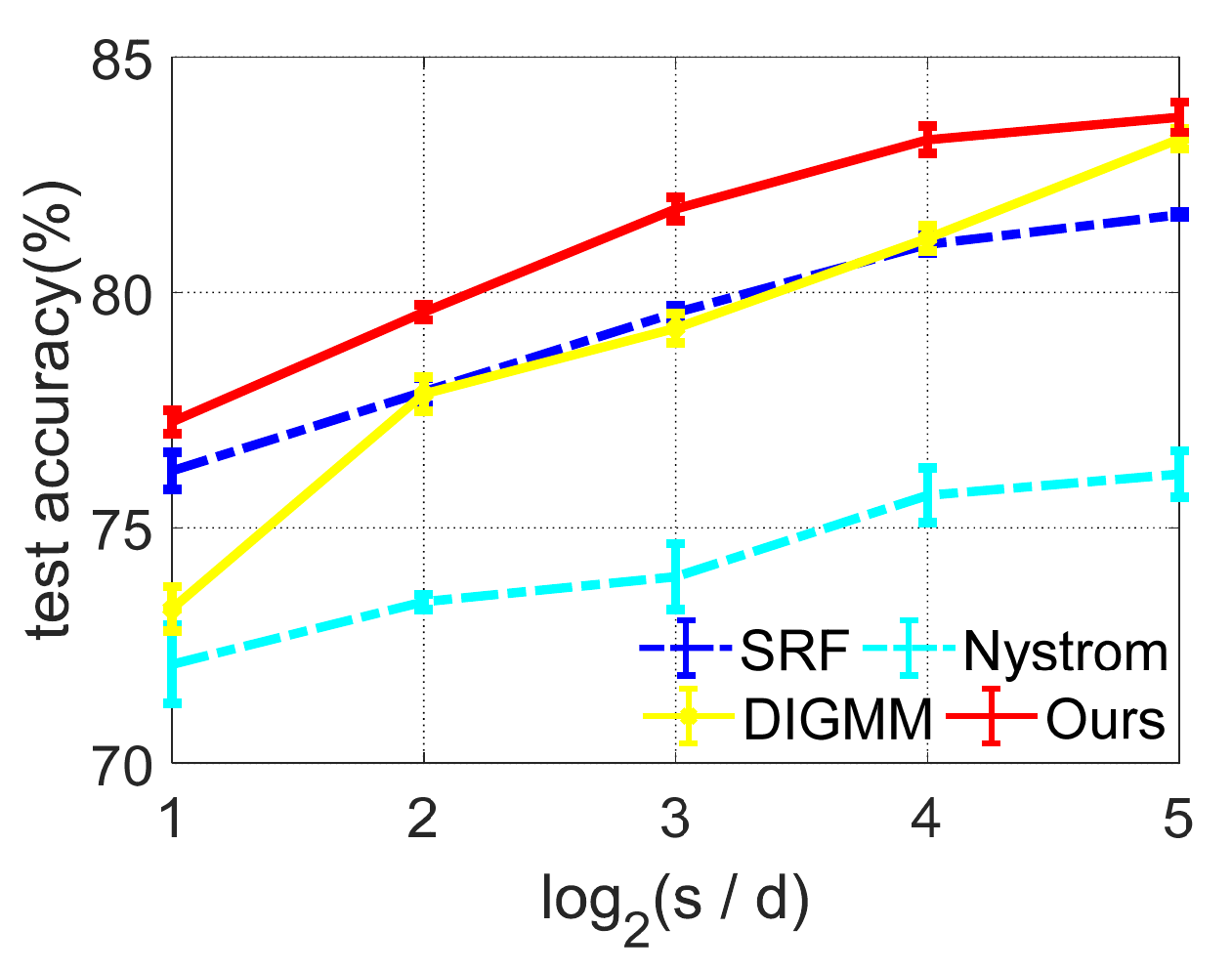}}
	\subfigure[\emph{cod-RNA}]{
		\includegraphics[width=0.236\textwidth]{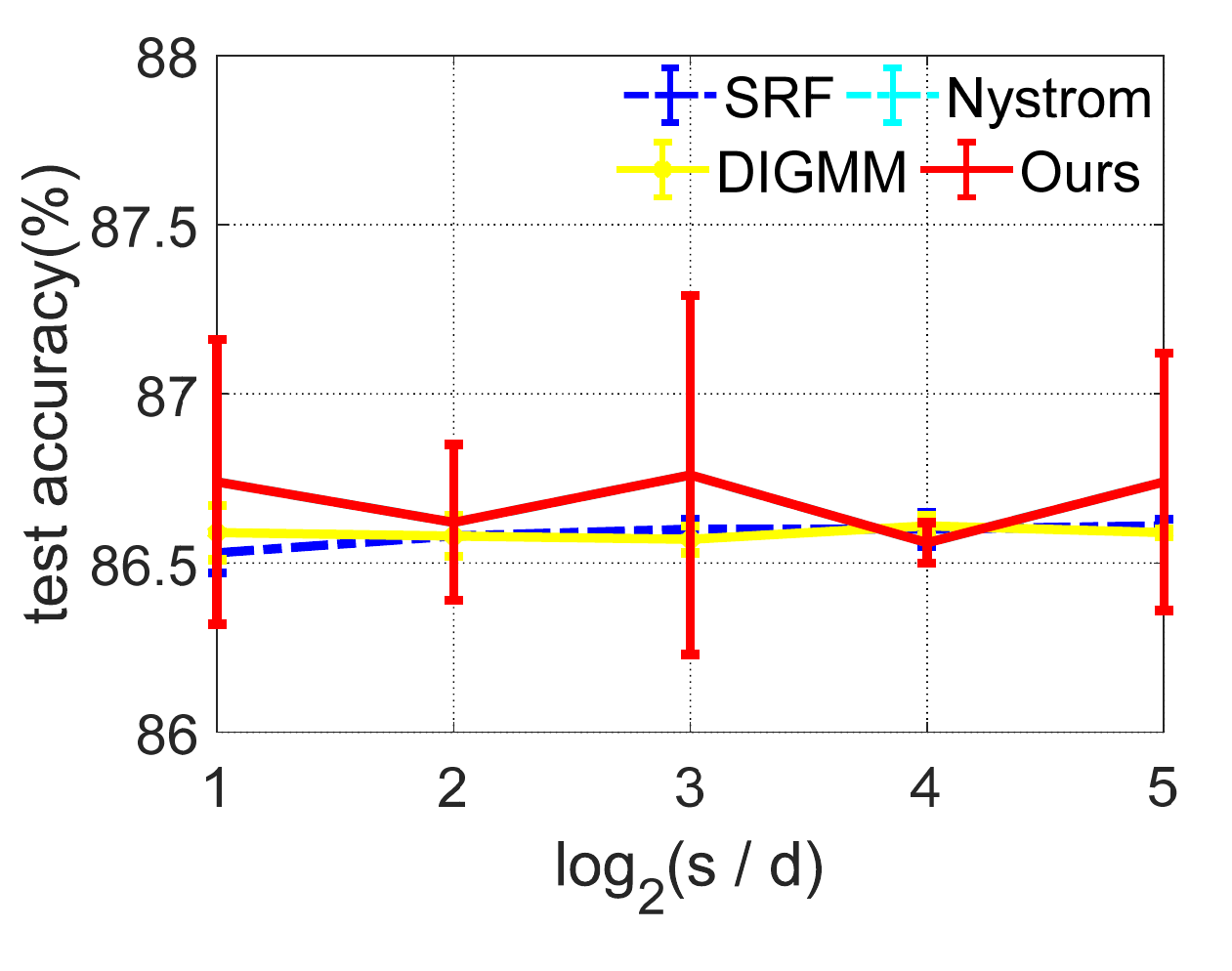}}
	
	\centering
	\subfigure{
		\includegraphics[width=0.234\textwidth]{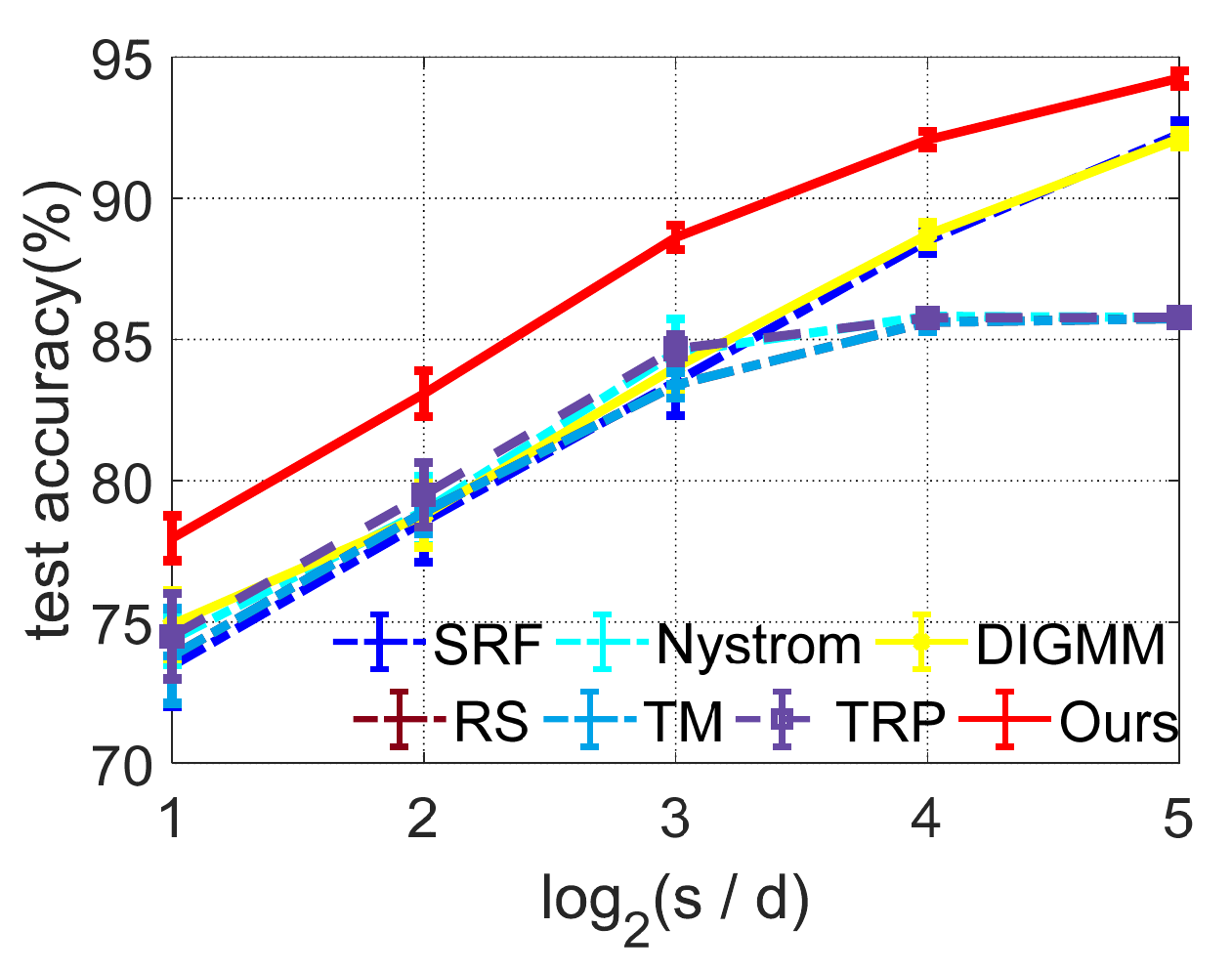}}
	\subfigure{
		\includegraphics[width=0.234\textwidth]{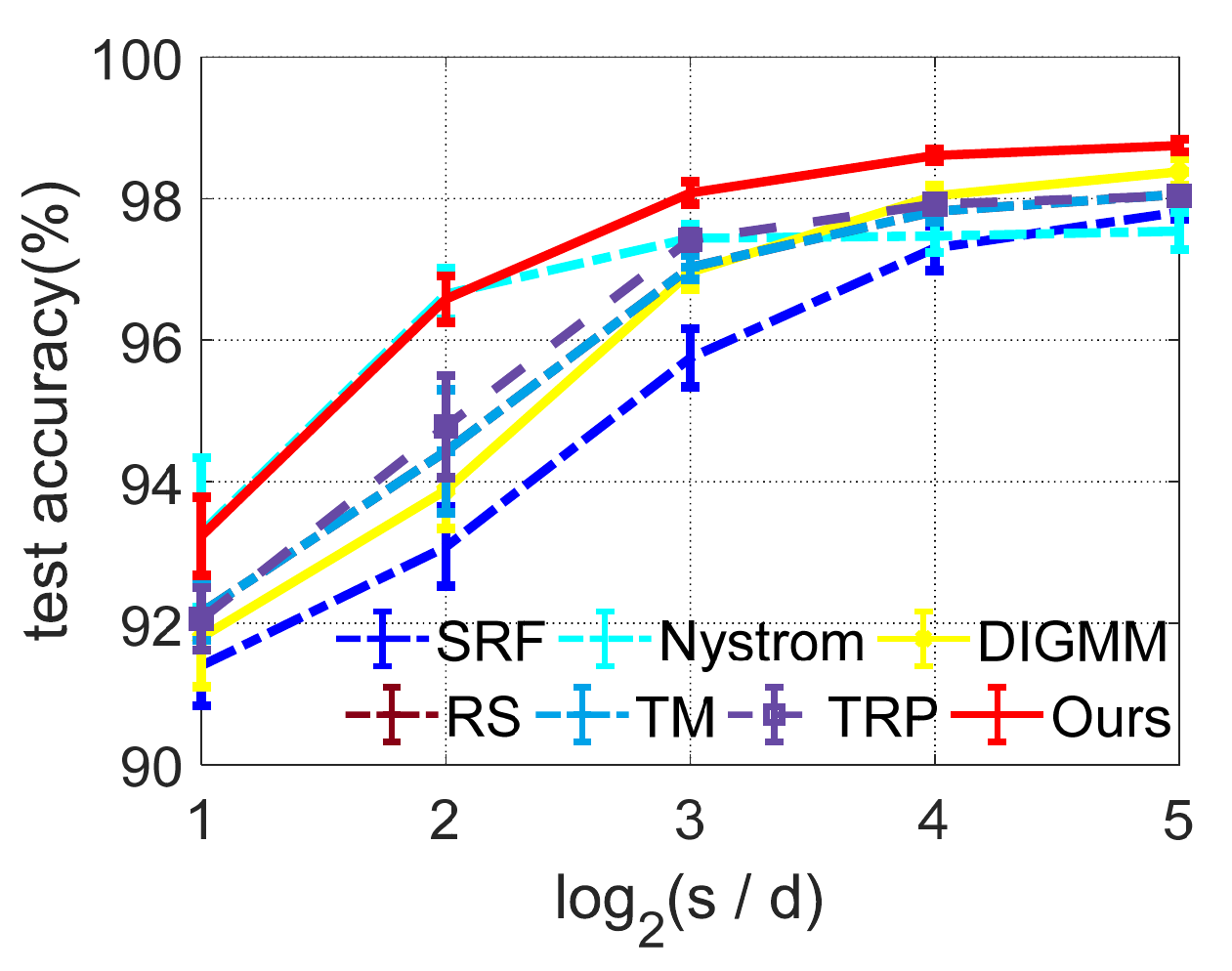}}
	\subfigure{
		\includegraphics[width=0.234\textwidth]{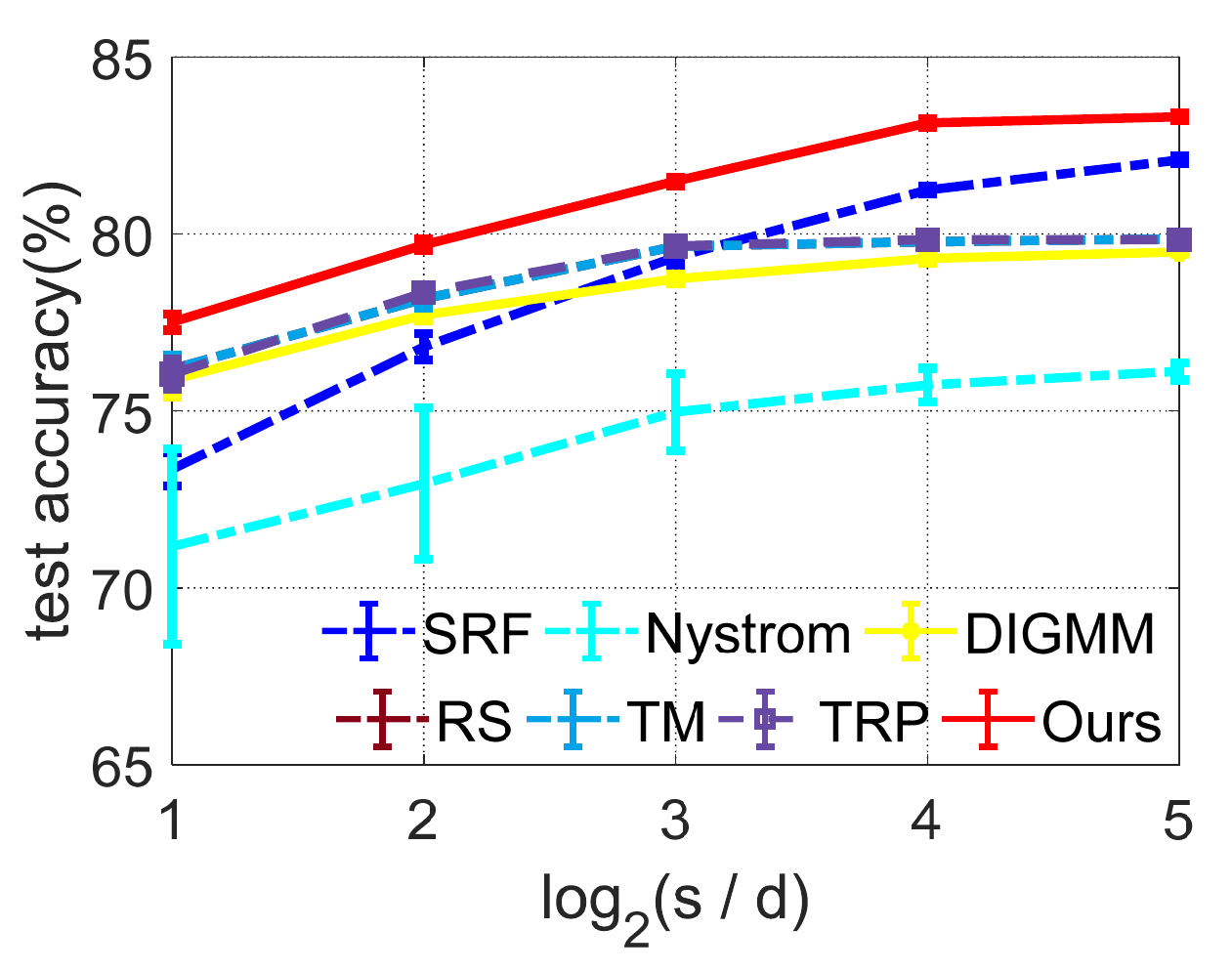}}
	\subfigure{
		\includegraphics[width=0.234\textwidth]{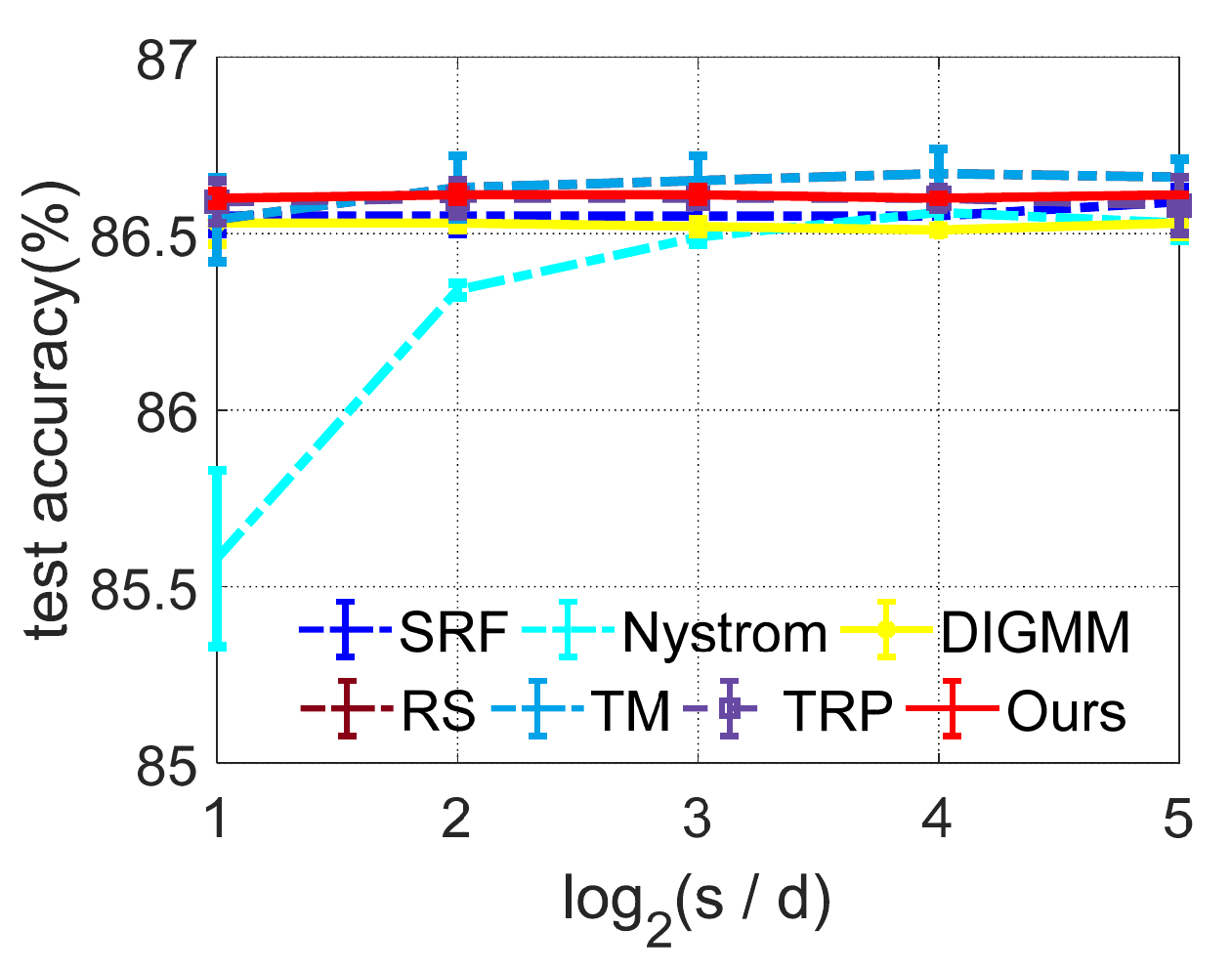}}
	
	\caption{Comparisons of various algorithms for classification accuracy with LibLinear across the Delta-Gaussian kernel (top) and the spherical polynomial kernel (down) on four datasets.}	\label{figacc}
	\vspace{-0.25cm}
\end{figure*}

{\bf Kernel approximation:}
The relative error $\| \bm K - \widetilde{\bm K} \|_\mathrm{F}/\| \bm K \|_{\mathrm{F}}$ is chosen to measure the approximation quality where $\bm K$ and $\widetilde{\bm K}$ denote the exact kernel matrix on 1,000 random selected samples and its approximated kernel matrix, respectively.
Figure~\ref{figapp} shows the approximation error under two indefinite kernels as a function of \#random features $s$.
Our method achieves lower approximation error than the other algorithms across such two kernels on these datasets in most cases.
A clear look at the case of Delta-Gaussian kernel approximation will find that our approach significantly improves the approximation quality compared to random features based algorithms: SRF and DIGMM.
There always exists a gap in SRF that uses a PD kernel to approximate an indefinite one since the negative part is overlooked. DIGMM only focuses on approximating a subset of the kernel matrix. Different from these two, our method directly approximates the indefinite kernel function by an unbiased estimator, which incurs no extra loss for kernel approximation.
Besides, when compared with several representative algorithms for polynomial kernels, e.g., RM, TS, and TRP, our method still performs well, which extends the application of our model.

\begin{table*}[t]
	\fontsize{8}{8}\selectfont
	\begin{threeparttable}
		\centering
		\caption{Time cost (sec.) for generating feature matrices of various algorithms.} 
		\label{tabtime}
		\begin{tabular}{ccccccccc|ccccccccccc}
			\toprule[1pt]
			\multirow{2}{1cm}{Datasets}&\multirow{2}{0.5cm}{\centering{$s$}}&\multicolumn{7}{c|}{Spherical polynomial kernel} &\multicolumn{4}{c}{Delta-Gaussian kernel}  \cr
			\cmidrule(lr){3-13}
			& &SRF\tnote{1} & Nystr\"{o}m & DIGMM & RM & TS & TRP & Ours &SRF\tnote{1} & Nystr\"{o}m & DIGMM & Ours \\
			\midrule
			\multirow{3}{1cm}{\emph{letter}} &$2d$ & {\bf 17.7}+0.02 & 0.12 & 0.21 & 0.04 & 0.07 & 0.01 & 0.08 & {\bf 17.2}+0.03 & 0.14 & 0.32 & 0.07 \\
			&$8d$ & 0.09 & 0.24 & 0.41 & 0.12 & 0.10 & 0.21& 0.23  & 0.10 & 0.39 & 1.19 & 0.23 \\
			&$32d$  & 0.33 & 1.27 & 2.69 & 0.38 & 0.31 & 0.89 & 0.85 & 0.30 & 1.69 & 4.56 & 0.92  \\
			\midrule
			\multirow{3}{1cm}{\emph{ijcnn1}} &$2d$ & {\bf 10.3}+0.24 & 0.83 & 0.42 & 0.33 & 0.53& 0.73 & 0.70 & {\bf 20.3}+0.23 & 1.23 & 0.61 & 0.41 \\
			&$8d$ & 0.89 & 2.78 & 1.08 & 1.20 & 0.90& 2.74 & 1.87 & 0.86 & 4.38 & 1.96 & 1.44 \\
			&$32d$  & 3.30 & 16.44 & 8.36 & 4.50 & 2.64& 10.50 & 7.31 & 3.42 & 22.64 & 6.86 & 5.67  \\
			\bottomrule
		\end{tabular}
		\begin{tablenotes}
			\footnotesize
			\item[1] On each dataset, SRF obtains parameters in GMM by an off-line grid search scheme in advance, of which this extra time cost is reported in {\bf bold}.
		\end{tablenotes}
	\end{threeparttable}
\end{table*}

{\bf Classification with linear SVM:} We train a linear classifier: LibLinear \cite{fan2008liblinear} with the obtained randomized feature map.
The balanced parameter in linear SVM is tuned by five-fold cross validation on a grid of points: $C = [0.01,0.1,1,10,100]$.
The test accuracy of various algorithms are shown in Figure~\ref{figacc}.
As we expected, higher-dimensional randomized feature map outputs higher classification accuracy except the \emph{cod-RNA} dataset.
On this dataset, all of algorithms achieve the similar classification accuracy under various $s$.
Apart from this dataset, our method performs best in most cases.

{\bf Computational time:}  
Table~\ref{tabtime} reports the time cost on generating randomized feature map with various dimensions $s$ on two datasets. 
Our method achieves the same complexity with the standard RFF with $\mathcal{O}(ns^2)$ time and $\mathcal{O}(ns)$ memory. In practice, as reported by Table~\ref{tabtime}, our method takes a little more time than SRF to generate randomized feature maps due to the introduced extra imaginary part.
Nevertheless, on each dataset, SRF requires extra time to obtain parameters of a sum of Gaussians in advance.

\begin{figure}[t]
	\centering
	\subfigure[approximation error]{
		\includegraphics[width=0.43\textwidth]{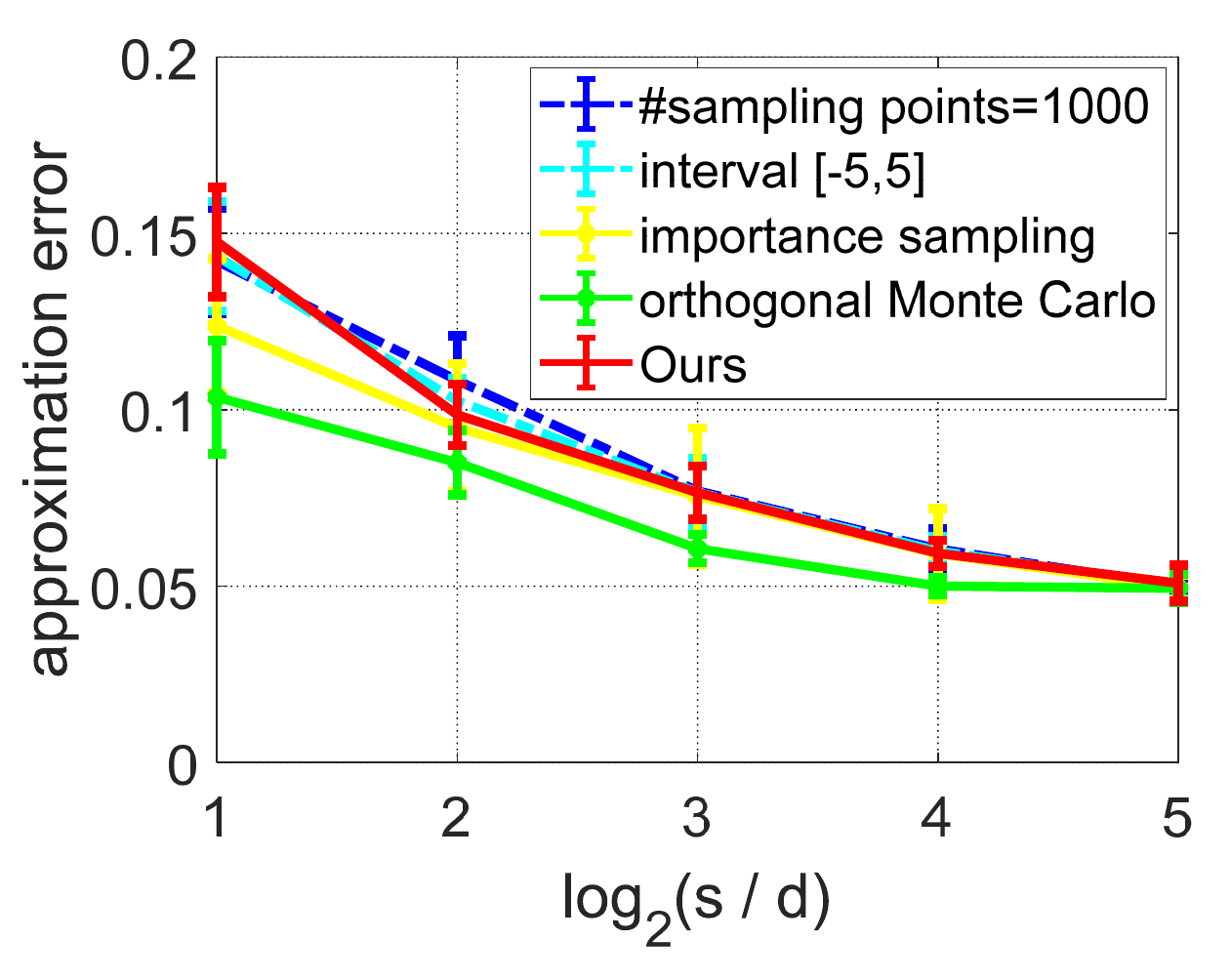}}
	\subfigure[time cost]{
		\includegraphics[width=0.43\textwidth]{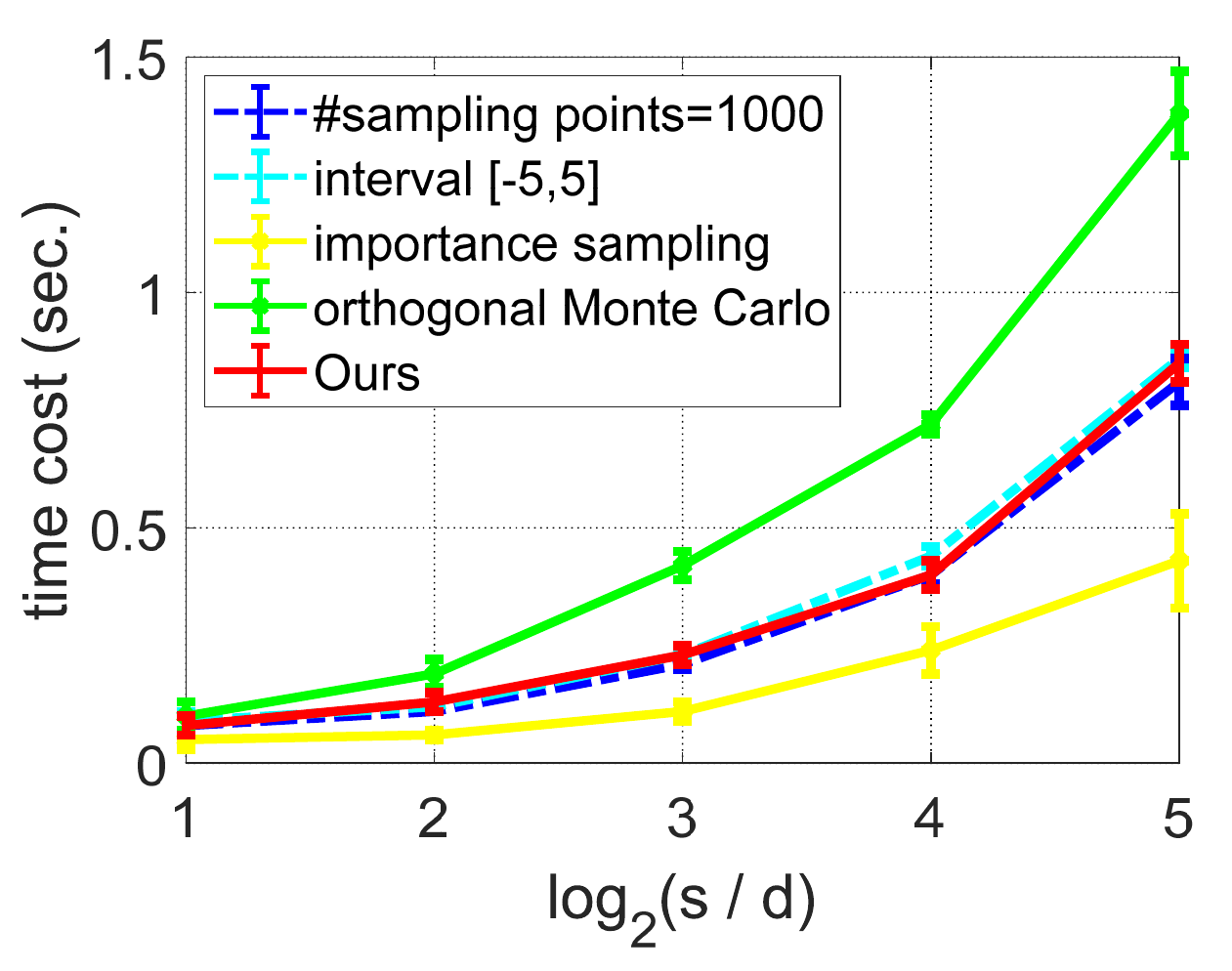}}
	\caption{Comparisons of various sampling schemes across the spherical polynomial kernel on \emph{letter}.}	\label{figsampling}
	\vspace{-0.25cm}
\end{figure}

{\bf Sampling schemes in spherical polynomial kernel:} 
The measures $\mu_+/\|\mu_+\|$ and $\mu_-/\| \mu_-\|$ associated with the spherical polynomial kernel are not typical distributions, so we conduct rejection sampling to acquire them by generating a set of uniformly 10,000 samples in a range of $[-10,10]$.
Here we compare various sampling schemes in our method for spherical polynomial kernel approximation, including sampling with 1,000 points, sampling in a sub-interval $[-5,5]$, importance sampling, and orthogonal Monte Carlo.
Here the \emph{surrogate} distribution in importance sampling is chosen as the Gaussian distribution.
The applied orthogonal Monte Carlo, followed by \cite{Yu2016Orthogonal}, aims to obtain orthogonal random features.

Figure~\ref{figsampling} shows the approximation error and time cost of various sampling schemes.
It can be found that, orthogonal random features achieve lower approximation error but require more computational cost, as suggested by \cite{Yu2016Orthogonal,choromanski2019unifying}.
Instead, the applied importance sampling decreases the time cost with a slight improvement on the approximation performance.
If we choose the ridge leverage function in importance sampling, our model works with the leverage score based sampling, refer to \cite{avron2017random} for details.

\begin{figure}[t]
	\centering
	\subfigure[order $p=1$]{
		\includegraphics[width=0.43\textwidth]{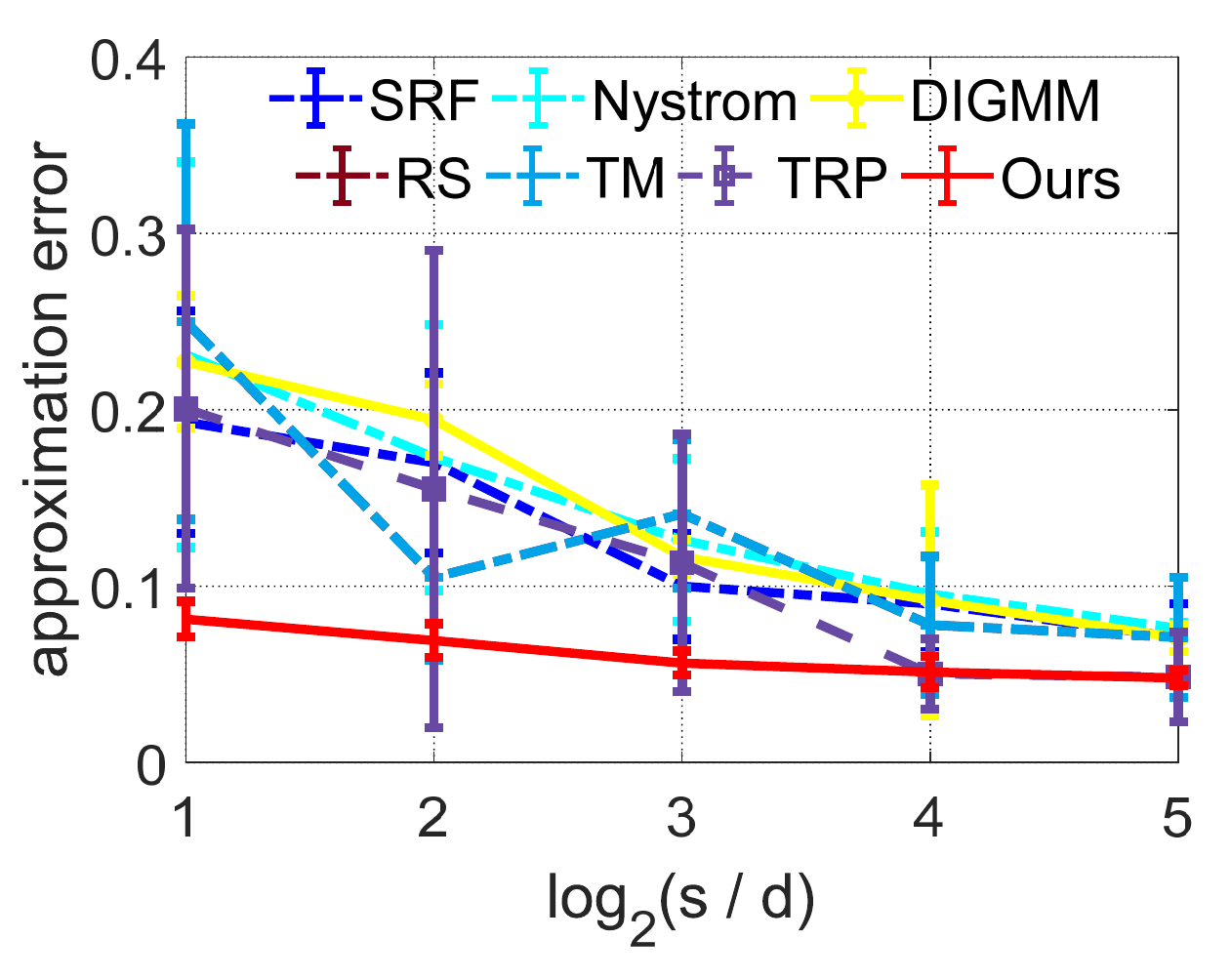}}
	\subfigure[order $p=3$]{
		\includegraphics[width=0.43\textwidth]{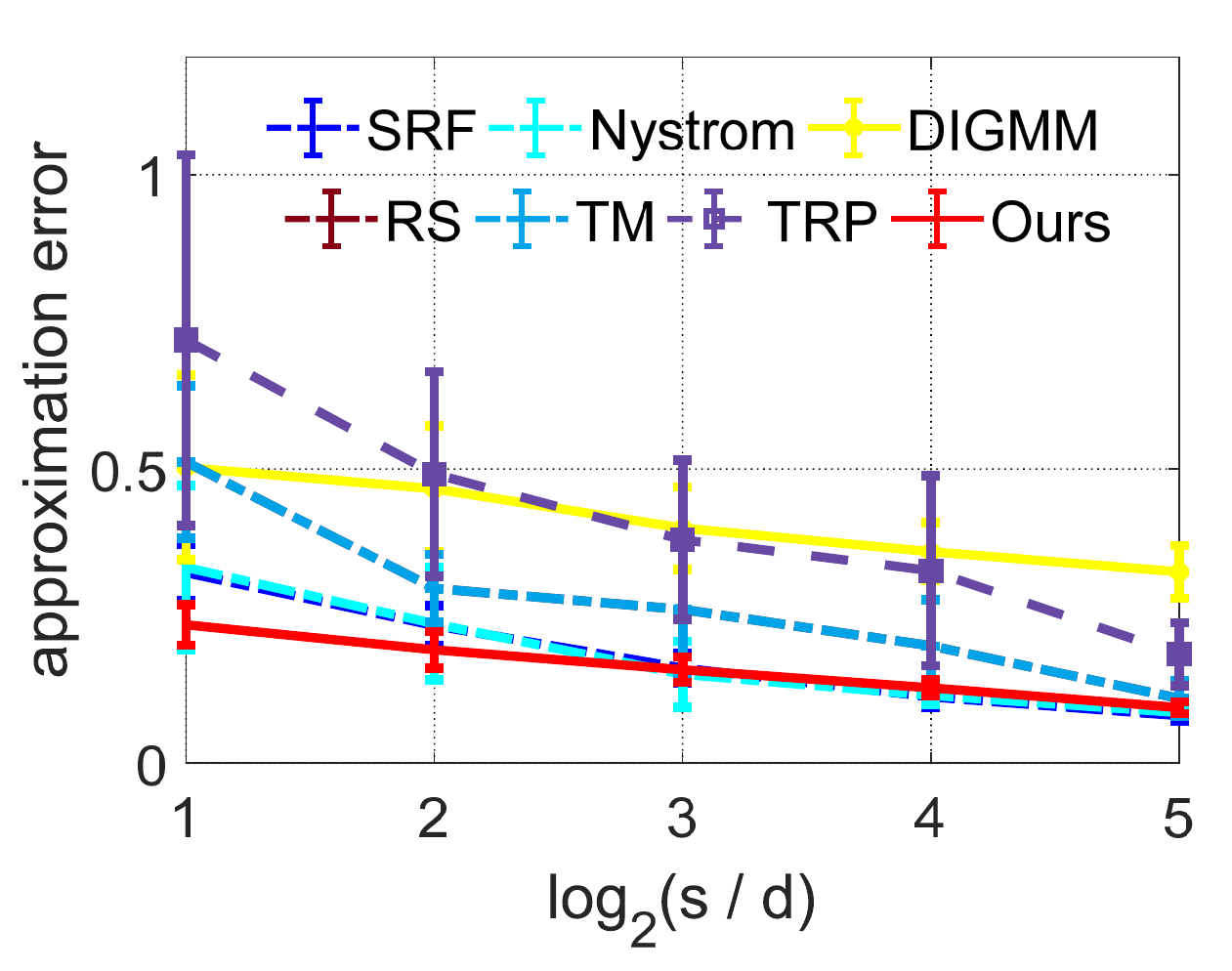}}
	\caption{Approximation error of various algorithms across the spherical polynomial kernel with different orders on the \emph{letter} dataset.}	\label{figorder}
	\vspace{-0.05cm}
\end{figure}

{\bf Different orders in spherical polynomial kernel:}
Apart from the used $p=2$ in the spherical polynomial kernel in our experiment, we evaluate our model on spherical polynomial kernels with various orders, e.g., $p=1$ and $p=3$. 
Kernel approximation results in Figure~\ref{figorder} show that, under different orders, our method performs better than other algorithms in terms of the approximation error.

\vspace{-0.3cm}
\section{Conclusion}
\vspace{-0.3cm}
We answer the open question of indefinite kernels in machine learning community by the introduced measure decomposition technique, which motivates us to develop a general random features algorithm across various kernels that are stationary indefinite kernels.
Albeit simple, our algorithm is effective to output unbiased estimates for indefinite kernel approximation.
Besides, our findings on the indefiniteness of NTK on the unit sphere (by $\ell_2$ normalization) encourages us to better understand the approximation performance, functional spaces, and generalization properties in over-parameterized networks in the future.

\section*{Acknowledgements}
The research leading to these results has received funding from the European Research Council under the European Union's Horizon 2020 research and innovation program / ERC Advanced Grant E-DUALITY (787960). This paper reflects only the authors' views and the Union is not liable for any use that may be made of the contained information.
This work was supported in part by Research Council KU Leuven: Optimization frameworks for deep kernel machines C14/18/068; Flemish Government: FWO projects: GOA4917N (Deep Restricted Kernel Machines: Methods and Foundations), PhD/Postdoc grant. This research received funding from the Flemish Government (AI Research Program). 
This work was supported in part by Ford KU Leuven Research Alliance Project KUL0076 (Stability analysis and performance improvement of deep reinforcement learning algorithms), EU H2020 ICT-48 Network TAILOR (Foundations of Trustworthy AI - Integrating Reasoning, Learning and Optimization), Leuven.AI Institute; and in part by the National Natural Science Foundation of China 61977046.


\newpage

\appendix
\onecolumn

Appendix is organized as follows.
\begin{itemize}
	\item Section~\ref{app:fft} gives the proof of Theorem~\ref{thm:fft} on the studied open question.
	\item The approximation performance is theoretically demonstrated in Section~\ref{app:error} in terms of uniform approximation error bound and variance reduction.
	\item In Section~\eqref{finsigm}, we demonstrate that polynomial kernels on the unit sphere by $\ell_2$ normalization data is with finite total mass. 
	\item Section~\ref{app:ftntk} gives the proof of Theorem~\ref{thentk}.
	\item The measure of arc-cosine kernels on the unit sphere by $\ell_2$ normalization data is given in Section~\ref{meazoa}.
\end{itemize}

\section{Proof of Theorem~\ref{thm:fft}}
\label{app:fft}
\begin{proof}
	We give the proof of the existence.\\
	\textit{(i)} Necessity.\\
	An stationary indefinite kernel associated with RKKS admits the positive decomposition
	\begin{equation*}
	k(\bm x - \bm x') = k_+(\bm x - \bm x') - k_-(\bm x - \bm x')\,,\quad \forall \bm x, \bm x' \in X\,,
	\end{equation*}
	where $k_+$ and $k_-$ are two positive definite kernels.
	According to the Bochner's theorem \cite{bochner2005harmonic}, there exists two probability measures $\mu_+$, $\mu_-$ such that 
	\begin{equation*}\label{intrep1}
	k(\bm z) = k_+(\bm z) - k_-(\bm z) = \int_{\Omega}  \exp \left(  \mathrm{i} {\bm \omega}^{\!\top} \bm z \right) \mu_+(\mathrm{d} {\bm \omega}) - \int_{\Omega}  \exp \left(  \mathrm{i} {\bm \omega}^{\!\top} \bm z \right) \mu_-(\mathrm{d} {\bm \omega})\,,
	\end{equation*}
	where $\bm z\coloneqq\bm x - \bm x'$.
	Denote $\mu\coloneqq \mu_+ - \mu_-$, it is clear that $\mu$ is a signed measure, and its total mass is finite because of $\| \mu \| = \| \mu_+ \| + \|\mu_-\| = 2$. 
	
	\textit{(ii)} Sufficiency.\\
	Let $\Omega\coloneqq\mathbb{R}^d$ and $\mathcal{A}$ be the smallest $\sigma$-algebra containing all open subsets of $\Omega$, and $\mu: \mathcal{A} \rightarrow [-\infty, \infty]$
	\begin{equation*}
	\mu (\bm \omega) = \int_{\Omega} \exp \left(-  \mathrm{i} {\bm \omega}^{\!\top} \bm z \right) k(\bm z) \mathrm{d} \bm z\,.
	\end{equation*}
	Since we assume that $\mu$ has total mass $\| \mu \| < \infty$, i.e., $\mu$ is finite, $\mu$ can be regarded as a signed measure. 
	By virtue of Jordan decomposition in Theorem~\ref{jorthem}, there exist two nonnegative finite measures $\mu_+$ and $\mu_-$ such that $\mu = \mu_+ - \mu_-$.
	One intuitive implementation way is choosing $\mu_+ = \max\{ \mu, 0 \}$ and $\mu_- = \min \{ 0, \mu \}$.
	Then using the inverse Fourier transform and Plancherel's theorem \cite{donoghue2014distributions}, we have
	\begin{equation*}
	\begin{split}
	k(\bm z) = \int_{\Omega}  \exp \left(  \mathrm{i} {\bm \omega}^{\!\top} \bm z \right) \mu(\mathrm{d} {\bm \omega}) &= \int_{\Omega}  \exp \left(  \mathrm{i} {\bm \omega}^{\!\top} \bm z \right) \mu_+(\mathrm{d} {\bm \omega}) - \int_{\Omega }  \exp \left(  \mathrm{i} {\bm \omega}^{\!\top} \bm z \right) \mu_-(\mathrm{d} {\bm \omega}) \\
	& = \| \mu_+ \| \int_{\Omega}  \exp \left(  \mathrm{i} {\bm \omega}^{\!\top} \bm z \right) \tilde{\mu}_+(\mathrm{d} {\bm \omega}) - \| \mu_- \| \int_{\Omega}  \exp \left(  \mathrm{i} {\bm \omega}^{\!\top} \bm z \right) \tilde{\mu}_-(\mathrm{d} {\bm \omega}) \\
	& = \| \mu_+ \| \tilde{k}_+ (\bm z) - \| \mu_- \| \tilde{k}_- (\bm z)\,,
	\end{split}
	\end{equation*}
	where $\tilde{\mu}_+\coloneqq\mu_+/\| \mu_+ \|$ and $\tilde{\mu}_-\coloneqq\mu_-/\| \mu_- \|$ are two nonnegative Borel measures, which correspond to two positive definite kernels $\tilde{k}_+$ and $\tilde{k}_-$, respectively.
	By defining $k_+\coloneqq\|\mu_+\|\tilde{k}_+$ and $k_-\coloneqq\|\mu_+\|\tilde{k}_-$, we have 
	\begin{equation*}
	k(\bm x, \bm x') = k_+(\bm x, \bm x') - k_-(\bm x, \bm x'),\quad \forall \bm x, \bm x' \in X\,.
	\end{equation*}
	This completes the proof. 
	
	Based on the above analysis, we give a characterization of the RKHSs $\mathcal{H}_{\pm}$ through the given spectral density $\mu_{\pm}$.
	In \cite{hotz2012representation}, a RKHS can be characterized by its measure via Fourier transform.
	Therefore, in our model, the RKHSs $\mathcal{H}_{\pm}$ are represented by $\mu_{\pm}$. That is, for any $f \in \mathcal{H}_{\pm}$, the inner product is induced by the Hilbert norm
	\begin{equation*}
	\|f\|_{\mathcal{H}_{\pm}}^{2}=\int_{\mathbb{R}^{d}} \frac{|F(\bm \omega)|^{2}}{\mu_+(\bm \omega)} \mathrm{d} \bm \omega \,,
	\end{equation*}
	where $F(\bm \omega) = \mathscr{F}(f) = \int_{\mathbb{R}^d} f(\bm x) e^{-2\pi \mathrm{i} \bm \omega^{\!\top} \bm x} \mathrm{d} \bm x$ is the Fourier transform of $f$.
	
\end{proof}

\section{Proof of Proposition~\ref{prop:error}}
\label{app:error}

The proof can be easily derived from \cite{sutherland2015error,choromanski2018geometry}, and we briefly present here for completeness.
\begin{proof}
	Proposition~1 in \cite{sutherland2015error} demonstrates 
	\begin{equation*}
	\operatorname{Pr}\!\left[ \sup_{\bm x, \bm x' \in \mathcal{S}_R} \! |k_{\pm}(\bm x, \bm x') \!-\! \tilde{k}_{\pm}(\bm x, \bm x')| \geq \epsilon \right]  \!\leq\! 66 \! \left(  \frac{\sigma_{\pm} R}{\epsilon}  \right)^{\!2} \exp\left( -\frac{s\epsilon^2}{8(d+2)}\right)\,,
	\end{equation*}
	where $\sigma_{\pm}^2 = \mathbb{E}_{\bm \omega \sim \tilde{\mu}_{\pm}}[{\bm \omega}^{\!\top} {\bm \omega}] < \infty$.
	Since the indefinite kernel $k$ admits
	\begin{equation*}
	|k(\bm x, \bm x') - \tilde{k}(\bm x, \bm x')|  \leq |k_{+}(\bm x, \bm x') - \tilde{k}_{+}(\bm x, \bm x')| + |k_{-}(\bm x, \bm x') \!-\! \tilde{k}_{-}(\bm x, \bm x')|\,,\quad \forall \bm x, \bm x' \in \mathcal{S}_R\,,
	\end{equation*} 
	then we have
	\begin{equation*}
	\begin{split}
	\operatorname{Pr}\!\left[ \sup_{\bm x, \bm x' \in \mathcal{S}_R} \! |k(\bm x, \bm x') \!-\! \tilde{k}(\bm x, \bm x')| \geq \epsilon \right] & \leq 	\operatorname{Pr}\!\left[ \sup_{\bm x, \bm x' \in \mathcal{S}_R} \! |k_+(\bm x, \bm x') \!-\! \tilde{k}_+(\bm x, \bm x')| \geq \frac{\epsilon}{2} \right] + 	\operatorname{Pr}\!\left[ \sup_{\bm x, \bm x' \in \mathcal{S}_R} \! |k_-(\bm x, \bm x') \!-\! \tilde{k}_-(\bm x, \bm x')| \geq \frac{\epsilon}{2} \right] \\
	& \leq 66 \! \left(  \frac{2\sigma_{+} R}{\epsilon}  \right)^{\!2} \exp\left( -\frac{s\epsilon^2}{32(d+2)}\right) + 66 \! \left(  \frac{2\sigma_{-} R}{\epsilon}  \right)^{\!2} \exp\left( -\frac{s\epsilon^2}{32(d+2)}\right) \\
	& = 66 \! \left(  \frac{2\sigma R}{\epsilon}  \right)^{\!2} \exp\left( -\frac{s\epsilon^2}{32(d+2)}\right)\,.
	\end{split}
	\end{equation*}
	
	Then we study the variance reduction of the applied orthogonal Monte Carlo (OMC) sampling.
	Based on the definition of $\operatorname{MSE}$, i.e., $\mathbb{E}[\tilde{k}(\bm x, \bm x')] = \mathbb{E}[k(\bm x, \bm x') - \tilde{k}(\bm x, \bm x')]$, we conclude that $\mathbb{E}[\tilde{k}(\bm x, \bm x')]$ is the variance of $\tilde{k}(\bm x, \bm x')$, termed as $\operatorname{Var}[\tilde{k}(\bm x, \bm x')]$ due to our unbiased estimator.
	According to Theorem~4.2 in \cite{choromanski2018geometry}, for sufficiently large $d$, we have
	\begin{equation*}
	\operatorname{MSE}[\tilde{k}_+^{\operatorname{OMC}}(\bm x, \bm x')] \leq \operatorname{MSE}[\tilde{k}_+^{\operatorname{MC}}(\bm x, \bm x')]\quad \mbox{and} \quad 	\operatorname{MSE}[\tilde{k}_-^{\operatorname{OMC}}(\bm x, \bm x')] \leq \operatorname{MSE}[\tilde{k}_-^{\operatorname{MC}}(\bm x, \bm x')]\,,
	\end{equation*}
	where $\bm \omega \sim \tilde{\mu}_+(\cdot)$ in $\tilde{k}_+$ and $\bm \nu \sim \tilde{\mu}_-(\cdot)$ in $\tilde{k}_-$ as indicated by Eq.~\eqref{lcintrep}.
	Since these two random vectors $\bm \omega$ and $\bm \nu$ are independent, we have
	\begin{equation*}
	\operatorname{Var}[\tilde{k}^{\operatorname{OMC}}(\bm x, \bm x')] = \operatorname{Var}[\tilde{k}_+^{\operatorname{OMC}}(\bm x, \bm x')] + \operatorname{Var}[\tilde{k}_-^{\operatorname{OMC}}(\bm x, \bm x')] \leq \operatorname{Var}[\tilde{k}_+^{\operatorname{MC}}(\bm x, \bm x')] + \operatorname{Var}[\tilde{k}_-^{\operatorname{MC}}(\bm x, \bm x')] = \operatorname{Var}[\tilde{k}^{\operatorname{MC}}(\bm x, \bm x')]\,.
	\end{equation*}
	which implies $\operatorname{MSE}[\tilde{k}^{\operatorname{OMC}}(\bm x, \bm x')] \leq \operatorname{MSE}[\tilde{k}^{\operatorname{MC}}(\bm x, \bm x')]$ for sufficiently large $d$.
	
\end{proof}

\section{Polynomial kernels on the unit sphere with finite total mass}
\label{finsigm}
We consider the asymptotic properties of the Bessel function of the first kind $J_{\alpha}(x)$ under the large and small cases to study the $\|\mu\|$.

\subsection{A small $\omega$}
Consider the asymptotic behavior for small $\omega$. The Bessel function of the first kind is asymptotically equivalent to
\begin{equation*}\label{Jsmall}
J_{\alpha}(x) \sim \frac{1}{\Gamma(\alpha +1)} \left( \frac{x}{2} \right)^{\alpha} \,,~\mbox{when}~0 < x \ll \sqrt{\alpha + 1} \,.
\end{equation*}
In this case, the measure $\mu$ is formulated as
\begin{equation}\label{smallw}
\mu(\omega) \sim \sum_{i=0}^{p} \frac{p !}{(p-i) !}\left(1-\frac{4}{a^{2}}\right)^{p-i}\left(\frac{2}{a^{2}}\right)^{i} \frac{2^{d / 2+i}}{\Gamma(d / 2+i +1)} \,,
\end{equation}
which can be regarded as a generalized version of a uniform distribution.
Therefore, $\mu$ is absolutely integrable over a finite range $(0,c_1]$, where $c_1$ is some constant satisfying $c_1 \ll \sqrt{\frac{d}{2}-1}$.

\subsection{A large $\omega$}
Consider the asymptotic behavior for large $\omega$. The Bessel function of the first kind is asymptotically equivalent to
\begin{equation*}\label{Jlarge}
J_{\alpha}(x) \sim \sqrt{\frac{2}{\pi x}} \cos(x-\frac{\pi \alpha}{2} - \frac{\pi}{4})\,,~\mbox{when}~x \gg |\alpha^2 - \frac{1}{4}| \,.
\end{equation*}
The Fourier transform of the polynomial kernel on the sphere, i.e., the measure $\mu$, is hence given by \cite{pennington2015spherical} 
\begin{equation}\label{largew}
\mu(\omega) \sim \frac{1}{\sqrt{\pi \omega}}\left(1-\frac{4}{a^{2}}\right)^{p}\left(\frac{2}{\omega}\right)^{d / 2} \cos \left((d+1) \frac{\pi}{4}-2 \omega\right)\,,~\mbox{for a large $\omega$}\,.
\end{equation}
In this way, we have $\int_{c_2}^{\infty} \left|\mu(\omega) \right|\mathrm{d}\omega < \infty$ for a large $\omega$, where $c$ is some constant satisfying $c_2 \gg \frac{1}{4}|d^2 - 1|$.

Accordingly, combining Eq.~\eqref{largew} with Eq.~\eqref{smallw}, we conclude that 
\begin{equation*}
\| \mu \| \coloneqq \int_{0}^{\infty} \left|\mu(\omega) \right|\mathrm{d}\omega = \int_{0}^{c_1} \left|\mu(\omega) \right|\mathrm{d}\omega + \int_{c_1}^{c_2} \left|\mu(\omega) \right|\mathrm{d}\omega + \int_{c_2}^{\infty} \left|\mu(\omega) \right|\mathrm{d}\omega < \infty \,,
\end{equation*}
where we use $\int_{c_1}^{c_2} \left|\mu(\omega) \right|\mathrm{d}\omega$ is finite due to the continuous, bounded Bessel function $J_{\alpha}(x)$ on a finite region $[c_1,c_2]$.

\section{Proof of Theorem~\ref{thentk}}
\label{app:ftntk}
To prove Theorem~\ref{thentk}, we firstly derive its formulation on the unit sphere and then demonstrate that it is a shift-invariant but not positive definite kernel via \emph{completely monotone} functions.

\begin{definition}\label{common} (Completely monotone \cite{schoenberg1938metric})
	A function $f$ is called completely monotone on $(0, +\infty)$ if it satisfies $f \in C^{\infty}(0,+\infty)$ and
	\begin{equation*}
	(-1)^r f^{(r)}(x) \geq 0 \,,
	\end{equation*}
	for all $r=0,1,2,\cdots$ and all $x > 0$. Moreover, $f$ is called completely monotone on $[0,+\infty)$ if it is additionally defined in $C[0,+\infty)$.
\end{definition}
Note that the definition of completely monotone functions can be also restricted to a finite interval, i.e., $f$ is completely monotone on $[a,b] \subset \mathbb{R}$, see in \cite{pennington2015spherical}.

Besides, we need the following lemma that demonstrates the connection between positive definite and completely monotone functions for the proof.
\begin{lemma}\label{scho} (Schoenberg's theorem \cite{schoenberg1938metric})
	A function $f$ is completely monotone on $[0, +\infty)$ if and only if $f\coloneqq g(\| \cdot \|_2^2)$ is radial and positive definite function on all $\mathbb{R}^d$ for every $d$.
\end{lemma}
Now let us prove Theorem~\ref{thentk}.
\begin{proof}
	By virtue of $\langle \bm x, \bm x' \rangle = 1-\frac{1}{2}\| \bm x - \bm x' \|_2^2$ and $\| \bm x \|_2 = \| \bm x' \|_2 = 1$, we have $\| \bm x - \bm x' \|_2 \in [0,2]$. Therefore, the standard NTK of a two-layer ReLU network can be formulated as
	\begin{equation*}
	\begin{split}
	k(\bm x, \bm x') & = \langle \bm x, \bm x' \rangle \kappa_0(\langle \bm x, \bm x' \rangle) + \kappa_1(\langle \bm x, \bm x' \rangle) \\
	& = \left(1-\frac{1}{2}\| \bm x - \bm x' \|_2^2 \right) \kappa_0\left(1-\frac{1}{2}\| \bm x - \bm x' \|_2^2 \right) + \kappa_1 \left(1-\frac{1}{2}\| \bm x - \bm x' \|_2^2\right) \\
	& = \frac{2-\| \bm x - \bm x' \|_2^2}{\pi} \arccos(\frac{1}{2}\| \bm x - \bm x' \|_2^2 -1) + \frac{\| \bm x - \bm x' \|_2}{2\pi} \sqrt{4-\| \bm x - \bm x' \|_2^2} \\
	& = \frac{2-z^2}{\pi} \arccos(\frac{1}{2}z^2 -1) + \frac{z}{2\pi} \sqrt{4-z^2}\,, z\coloneqq\| \bm x - \bm x' \|_2 \in [0,2]\,,
	\end{split}
	\end{equation*}
	which is shift-invariant.
	
	Next, we prove that $k(z)$ is not a positive definite kernel, i.e., $g(\sqrt{z})\coloneqq k(z)$ is not a completely monotone function over $[0,\infty)$ by Lemma~\ref{scho}.
	In other words, there exist some value $x \in [0, \infty)$ such that $(-1)^l g^{(l)}(x) < 0$ for some $l$.
	To this end, the function $g$ is given by
	\begin{equation*}
	g(x) = \frac{2-x}{\pi} \arccos(\frac{1}{2}x -1) + \frac{1}{2\pi} \sqrt{4x-x^2}\,, ~x \in [0,4]\,,
	\end{equation*}
	and its first-order derivative is
	\begin{equation*}
	g'(x) =	\dfrac{4-2x}{4{\pi}\sqrt{4x-x^2}}-\dfrac{2-x}{2{\pi}\sqrt{1-\left(\frac{x}{2}-1\right)^2}}-\dfrac{\arccos\left(\frac{x}{2}-1\right)}{{\pi}}\,.
	\end{equation*}
	Since $g'(x)$ is continuous, and $\lim_{x\rightarrow 0}g'(x) = -\infty$ and $\lim_{x\rightarrow 4}g'(x) = \infty$, there exists a constant $c$ such that $g'(x) < 0$ over $(0,c)$ and $g'(x) > 0$ over $(c,4)$.
	That is to say, $(-1)^l g^{(l)}(x) < 0$ holds for $x \in (c,4)$, which violates the definition of completely monotone functions.
	In this regard, $g(\sqrt{z})\coloneqq k(z)$ is not a completely monotone function over $[0,\infty)$ and thus $\{k(z), z \in [0,2]; 0, z > 2 \}$ is not positive definite. 
\end{proof}

\section{The measure of arc-cosine kernels on the unit sphere}
\label{meazoa}

According to Appendix~\ref{app:ftntk}, the zero/first-order arc-cosine kernel on the unit sphere is proven to be stationary but indefinite.
In this section, we derive its measure $\mu$.

\subsection{The measure of the zero-order arc-cosine kernel}
In this section, we derive the measure $\mu$ of the zero-order arc-cosine kernel on the unit sphere.

\begin{proposition}
	The measure $\mu$ of the zero-order arc-cosine kernel on the unit sphere: $\kappa_0(\bm x, \bm x') := \kappa_0(z) = \frac{1}{\pi} \arccos(\frac{1}{2}z^2-1)$ is given by
	\begin{equation*}
	\begin{split}
	\mu(\omega) = (\frac{1}{\omega}) (\frac{2}{\omega})^{\frac{d}{2}-1} J_{\frac{d}{2}}(2\omega) - \frac{1}{\pi} (\frac{1}{\omega})^{\frac{d}{2}-2}  \sum_{j=0}^{\infty} \frac{(2 j) !}{4^{j}(j !)^{2}(2 j+1)} \int_{0}^2 \left(\frac{1}{2}z^2-1\right)^{2 j+1} \omega z^{d / 2} J_{d / 2-1}(z \omega) \mathrm{d}z \,,
	\end{split}
	\end{equation*}
	where the integral $\int_{0}^2 \left(\frac{1}{2}z^2-1\right)^{2 j+1} \omega z^{d / 2} J_{d / 2-1}(z \omega) \mathrm{d}z $ can be computed by parts with the following simple recurrence formula
	\begin{equation}\label{receJbz}
	\int z^{a} J_{v+1}(z) \mathrm{d} z = 2 v \int z^{a-1} J_{v}(z) \mathrm{d} z -\int z^{a} J_{v-1}(z) \mathrm{d} z\,.
	\end{equation}
\end{proposition}

\begin{proof}
	According to the definition of $\kappa_0(z)$, we have
	\begin{equation}\label{intzero}
	\begin{aligned}
	\mu(\omega) = \int_{0}^{2} \frac{z}{\pi} \arccos(\frac{1}{2}z^2-1) (z / \omega)^{d / 2-1} J_{d / 2-1}(z \omega) \mathrm{d}z\,,
	\end{aligned}\end{equation}
	where $\kappa_0(\bm z)$ is a radial function, i.e., $\kappa_0(\bm z) = \kappa_0(z)$ with $z\coloneqq\| \bm z \|_2$, and thus its Fourier transform is also a radial function, i.e., $\mu(\omega) = \mu(\bm \omega)$ with $\omega\coloneqq\| \bm \omega \|_2$.
	Obviously, the integrand in Eq.~\eqref{intzero} and the integration region are both bounded, and thus we have $\mu(\omega)<\infty$. 
	Following the proof of $\| \mu \| < \infty$ for polynomial kernels on the unit sphere in Section~\ref{finsigm}, we can also demonstrate that $\| \mu \| < \infty$ for the zero-order arc-cosine kernel on the unit sphere.
	
	To compute the integration in Eq.~\eqref{intzero}, we take the Taylor expansion of $\arccos(\frac{1}{2}z^2-1)$ with $t$ terms
	\begin{equation*}\begin{aligned}
	\arccos(\frac{1}{2}z^2-1) &
	=\frac{\pi}{2}-\sum_{j=0}^{t} \frac{(2 j) !}{4^{j}(j !)^{2}(2 j+1)} \left(\frac{1}{2}z^2-1\right)^{2 j+1}\,,
	\end{aligned}\end{equation*}
	and thus the integration in Eq.~\eqref{intzero} can be integrated by each term regarding to Bessel functions.
	Moreover, by virtue of $\frac{\mathrm{d} z^v J_v(z \omega)}{\mathrm{d} z} = \omega z^v J_{v-1}(z \omega)$, the above integral can be computed by parts
	\begin{equation}\label{zd1}
	\begin{split}
	\mu(\omega) &= \int_{0}^{2} \frac{z}{\pi} \arccos(\frac{1}{2}z^2-1) (z / \omega)^{d / 2-1} J_{d / 2-1}(z \omega) \mathrm{d}z\\
	& = \frac{1}{2} (\frac{1}{\omega})^{\frac{d}{2}-2} \int_0^2 \omega z^{\frac{d}{2}} J_{\frac{d}{2}-1}(z \omega) \mathrm{d} z - \frac{1}{\pi} (\frac{1}{\omega})^{\frac{d}{2}-2}  \sum_{j=0}^{\infty} \frac{(2 j) !}{4^{j}(j !)^{2}(2 j+1)} \int_{0}^2 \left(\frac{1}{2}z^2-1\right)^{2 j+1} \omega z^{d / 2} J_{d / 2-1}(z \omega) \mathrm{d}z \,,
	\end{split}
	\end{equation}
	where the first term equals to $(\frac{1}{\omega}) (\frac{2}{\omega})^{\frac{d}{2}-1} J_{\frac{d}{2}}(2\omega) $.
	Accordingly, we can conclude our proof.
\end{proof}

It appears non-trivial to prove $\| \mu \| < \infty$ as Eq.~\eqref{zd1} is quite complex.
Here we choose $j=0$ in Eq.~\eqref{zd1} as an example, we have
\begin{equation}\label{zd2}
\begin{split}
&\int_{0}^2 \left(\frac{1}{2}z^2-1\right) \omega z^{d / 2} J_{d / 2-1}(z \omega) \mathrm{d}z = 2^{\frac{d}{2}} J_{\frac{d}{2}}(2\omega) - \int_0^2 z^{\frac{d}{2}+1} J_{\frac{d}{2}} (z \omega) \left(\frac{1}{2}z^2 -1\right) \mathrm{d} z \\& 
\quad = 2^{\frac{d}{2}} J_{\frac{d}{2}}(2\omega) + \frac{1}{\omega}J_{\frac{d}{2}+1}(2\omega) - \frac{1}{2}\int_0^2 z^{\frac{d}{2}+3} J_{\frac{d}{2}} (z \omega)\mathrm{d} z\,,
\end{split}
\end{equation}
where $\int_0^2 z^{\frac{d}{2}+3} J_{\frac{d}{2}} (z \omega)\mathrm{d} z$ can be computed by parts
\begin{equation}\label{zd3}
\begin{split}
\int_0^2 z^{\frac{d}{2}+3} J_{\frac{d}{2}} (z \omega)\mathrm{d} z = 2^{\frac{d}{2}+3} J_{\frac{d}{2}}(2\omega) - \frac{1}{\omega^2} 2^{\frac{d}{2}+2} J_{\frac{d}{2}+2}(2\omega)\,.
\end{split}
\end{equation}
Incorporating Eqs.~\eqref{zd3},~\eqref{zd2} into Eq.~\eqref{zd1}, we have
\begin{equation*}
\begin{split}
\mu(\omega) &= \left(\frac{1}{\omega}\right) \left(\frac{2}{\omega}\right)^{\frac{d}{2}-1} J_{\frac{d}{2}}(2\omega) - \frac{1}{\pi} (\frac{1}{\omega})^{\frac{d}{2}-2} \left[ (-3)2^{\frac{d}{2}}J_{\frac{d}{2}}(2\omega) + \frac{1}{\omega} J_{\frac{d}{2}+1} (2\omega) + \frac{1}{\omega^2} 2^{\frac{d}{2}+1} J_{\frac{d}{2}+2} (2\omega) \right]\,.
\end{split}
\end{equation*}
Following with the proof in Section~\ref{finsigm}, we can demonstrate $\| \mu \| < \infty$ by the asymptotic equivalence of Bessel functions.
Accordingly, in this case, $\mu$ can be decomposed into two nonnegative measures with $\mu(\omega) = \mu_+(\omega) - \mu_-(\omega)$, where $\mu_+(\omega) = \max\{ 0, \mu(\omega) \}$ and $\mu_-(\omega) = \max\{ 0, -\mu(\omega) \}$.
As a consequence, Algorithm~\ref{alg:one1} is also suitable for this kernel.

\subsection{the first-order arc-cosine kernel}
\label{meafir}
In this subsection, we derive the measure $\mu$ of the zero-order arc-cosine kernel admitting $\kappa_1(\bm x, \bm x') = \frac{z}{2\pi} \sqrt{4-z^2}$.

\begin{proposition}
	The measure $\mu$ of the zero-order arc-cosine kernel on the unit sphere: $\kappa_0(\bm x, \bm x') := \kappa_1(z) = \frac{z}{2\pi} \sqrt{4-z^2}$ is given by
	\begin{equation*}
	\begin{split}
	\mu(\omega) = (\frac{1}{\omega}) (\frac{2}{\omega})^{\frac{d}{2}-1} J_{\frac{d}{2}}(2\omega) - \frac{1}{\pi} (\frac{1}{\omega})^{\frac{d}{2}-2}  \sum_{j=0}^{\infty} \frac{(2 j) !}{4^{j}(j !)^{2}(2 j+1)} \int_{0}^2 \left(\frac{1}{2}z^2-1\right)^{2 j+1} \omega z^{d / 2} J_{d / 2-1}(z \omega) \mathrm{d}z \,,
	\end{split}
	\end{equation*}
	where the integral $\int_{0}^2 \left(\frac{1}{2}z^2-1\right)^{2 j+1} \omega z^{d / 2} J_{d / 2-1}(z \omega) \mathrm{d}z $ can be computed by parts with the following simple recurrence formula~\eqref{receJbz}.
\end{proposition}

\begin{proof}
	By fractional binomial theorem, we have	
	
	\begin{equation*}\begin{aligned}
	\left(\begin{array}{c}
	1 / 2 \\
	j
	\end{array}\right)
	=&(-1)^{k-1} \frac{1}{2(2j-1)} \frac{(2 j) !}{(2 \cdot 4 \cdot \cdot(2 j))^{2}} = \frac{-1}{2(2j-1)} \left(-\frac{1}{4}\right)^{j}\left(\begin{array}{c}
	2 j \\
	j
	\end{array}\right)\,.
	\end{aligned}\end{equation*}
	Then, according to the definition of $\kappa_1(z)$, we have
	\begin{equation*}
	\sqrt{4-z^2} = 2\left( 1 - \frac{z^2}{4} \right)^{\frac{1}{2}} =2\sum_{j=0}^{\infty}\left(\begin{array}{l}
	1/2 \\
	j
	\end{array}\right) \left( - \frac{z^2}{4} \right)^{j} = \sum_{j=0}^{\infty} \frac{-1}{2j-1} \left(\begin{array}{l}
	2j \\
	j
	\end{array}\right) \left(  \frac{z}{4} \right)^{2j}\,.
	\end{equation*}

	Therefore, the measure $\mu$ of $\kappa_1$ is
	\begin{equation}\label{muarcfir}
	\begin{aligned}
	\mu(\omega) &= \frac{1}{2\pi} \int_{0}^{2} z^2 \sqrt{4-z^2} (z / \omega)^{d / 2-1} J_{d / 2-1}(z \omega) \mathrm{d}z \\
	& = \frac{1}{2\pi} \int_{0}^{2} z^2 (z / \omega)^{d / 2-1} J_{d / 2-1}(z \omega) \sum_{j=0}^{\infty} \frac{-1}{2j-1} \left(\begin{array}{l}
	2j \\
	j
	\end{array}\right) \left(  \frac{z}{4} \right)^{2j} \mathrm{d}z.
	\end{aligned}\end{equation}
	Accordingly, the above equation needs to compute the following integral
	\begin{equation*}
	\int_{0}^2 z^{\frac{d}{2} + 1 + 2j} J_{\frac{d}{2} -1}(z \omega) \mathrm{d} z\,,
	\end{equation*}
	which can be computed by Eq.~\eqref{receJbz}.
\end{proof}

Similarly, it appears non-trivial to prove $\| \mu \| < \infty$ as Eq.~\eqref{muarcfir} is quite complex.
Here we choose $j=0$ in Eq.~\eqref{muarcfir} as an example, we have
\begin{equation*}
\mu(\omega) = \frac{1}{2\pi} \int_{0}^{2} z^2 (z / \omega)^{d / 2-1} J_{d / 2-1}(z \omega) \mathrm{d}z
= \frac{\sqrt{2}-1}{2\pi} \left( \frac{1}{\omega} \right)^{\frac{d}{2}-2} 2^{\frac{d}{2}} J_{\frac{d}{2}} (2 \omega)\,.
\end{equation*}
In this case, it is clear that $\| \mu \|<\infty$ and thus Algorithm~\ref{alg:one1} is also suitable for this kernel.

\end{document}